\documentclass[11pt]{article} % For LaTeX2e
\usepackage{fullpage}

\usepackage{comment}
\usepackage{url}
\usepackage{amsmath}
\usepackage{graphicx,xspace}
\usepackage{microtype}
\usepackage{booktabs} % for professional tables
\usepackage{amsfonts}       % blackboard math symbols
\usepackage{nicefrac}       % compact symbols for 1/2, etc.
\usepackage{microtype}      % microtypography

\usepackage{caption}
\usepackage{subcaption}
\usepackage{color}
\usepackage{mdwlist}

\usepackage{paralist}
\usepackage{algorithm,algorithmic}
\newfloat{Algorithm}{tb}{lox}
\floatname{Algorithm}{Algorithm}

\usepackage[colorlinks,linkcolor=red,citecolor=blue,urlcolor=blue]{hyperref}
\usepackage{caption}
\usepackage{subcaption}
\DeclareCaptionType{copyrightbox}
\usepackage{tabularx}
\usepackage{etoolbox}
\usepackage{mdwlist}
\usepackage{breakurl}

\usepackage[numbers,sort&compress]{natbib}

\usepackage{graphicx} % more modern
\usepackage{tikz}
\usetikzlibrary{positioning,patterns,shapes, shadows}
\usepackage{pgfplots}
\usepgfplotslibrary{external}
\usepackage{multirow}
\usepackage{rotating}
\usepackage{wrapfig}
\usepackage{array} 
\usepackage{pbox}

\colorlet{darkgreen}{green!50!black}
\colorlet{darkred}{red!90!black}
\colorlet{darkyellow}{yellow!90!black}

\tikzset{
  gnode/.style={
    fill=white,
    draw=black,
    circle,
    very thick, 
    inner sep=3.5,
    drop shadow={shadow xshift=0.3ex,shadow yshift=-0.5ex, path
      fading={circle with fuzzy edge 20 percent}}
  }
}

\tikzset{
  rnode/.style={
    fill=black,
    draw=black,
    circle,
    very thick, 
    inner sep=3.5,
    drop shadow={shadow xshift=0.3ex,shadow yshift=-0.5ex, path
      fading={circle with fuzzy edge 20 percent}}
  }
}

\tikzset{
  ynode/.style={
    fill=black!50!white,
    draw=black,
    circle,
    very thick, 
    inner sep=3.5,
    drop shadow={shadow xshift=0.3ex,shadow yshift=-0.5ex, path
      fading={circle with fuzzy edge 20 percent}}
  }
}

\usepackage{amssymb}% http://ctan.org/pkg/amssymb
\usepackage{pifont}% http://ctan.org/pkg/pifont
\usepackage{natbib}
\usepackage{algorithm}
\usepackage{algorithmic}

\input{./Definitions}
\graphicspath {{./Figures/}}

\begin{document}

%%%%%% Definitions for NOMAD diagrams %%%%%%%%%%%%%%%%%%
\def\gridwidth{2}
  \def\nodenum{4}
  \def\asnwidth{0.5}
  \def\asnmargin{0.05}

\newcommand{\asnbox}[3]{
  \fill [#3!25!lightgray]
  (#1 * \asnwidth + \asnmargin, 
  \gridwidth * 3 - #2 * \gridwidth + \asnmargin)
  rectangle
  (#1 * \asnwidth + \asnwidth - \asnmargin, 
  \gridwidth * 3 - #2 * \gridwidth + \gridwidth - \asnmargin);
  \draw
  (#1 * \asnwidth + \asnmargin, 
  \gridwidth * 3 - #2 * \gridwidth + \asnmargin)
  rectangle
  (#1 * \asnwidth + \asnwidth - \asnmargin, 
  \gridwidth * 3 - #2 * \gridwidth + \gridwidth - \asnmargin);
}

  \newcommand{\dasnbox}[2]{
    \draw[dashed]
    (#1 * \asnwidth + \asnmargin, 
    \gridwidth * 3 - #2 * \gridwidth + \asnmargin)
    rectangle
    (#1 * \asnwidth + \asnwidth - \asnmargin, 
    \gridwidth * 3 - #2 * \gridwidth + \gridwidth - \asnmargin);
  }

  \newcommand{\dtpt}[2]{
    \draw (#1 * \asnwidth + 0.5 * \asnwidth, 
    #2 * \asnwidth + 0.5 * \asnwidth) 
    node {\tiny{x}};
  }

  \newcommand{\mvarr}[4]{
    \draw [->]
    (#1 * \asnwidth + 0.5 * \asnwidth, 
    3 * \gridwidth - #2 * \gridwidth + 0.5 * \gridwidth) 
    --
    (#3 * \asnwidth + 0.5 * \asnwidth, 
    3 * \gridwidth - #4 * \gridwidth + 0.5 * \gridwidth);
  }
  
  \newcommand{\boxlayout}{
    % outer box
    \draw (0,0) rectangle (\gridwidth * \nodenum, \gridwidth * \nodenum);

    \fill [red!10] (0, 3 * \gridwidth) 
    rectangle (\gridwidth * \nodenum, 4 * \gridwidth);

    \fill [green!10] (0, 2 * \gridwidth) 
    rectangle (\gridwidth * \nodenum, 3 * \gridwidth);

    \fill [blue!10] (0, 1 * \gridwidth) 
    rectangle (\gridwidth * \nodenum, 2 * \gridwidth);

    \fill [brown!10] (0,0) 
    rectangle (\gridwidth * \nodenum, 1 * \gridwidth);

    % split between nodes
    \foreach \y in {1,2,3}
    {
      \draw[densely dashed] 
      (-\gridwidth * 0.1, \gridwidth * \y)
      --
      (\gridwidth * \nodenum + \gridwidth * 0.1,  \gridwidth * \y);
    }

    \dtpt{5}{1}  \dtpt{0}{15}  \dtpt{13}{12}  \dtpt{13}{6}  \dtpt{14}{8}  \dtpt{0}{8}  \dtpt{7}{10}  \dtpt{15}{2}  \dtpt{10}{5}  \dtpt{8}{13}  \dtpt{5}{3}  \dtpt{12}{5}  \dtpt{10}{7}  \dtpt{11}{6}  \dtpt{15}{3}  \dtpt{6}{9}  \dtpt{13}{5}  \dtpt{1}{3}  \dtpt{8}{15}  \dtpt{1}{6}  \dtpt{1}{7}  \dtpt{8}{9}  \dtpt{12}{4}  \dtpt{12}{1}  \dtpt{13}{10}  \dtpt{9}{0}  \dtpt{8}{4}  \dtpt{15}{12}  \dtpt{3}{6}  \dtpt{10}{13} 
  }

  \newcommand{\ptrna}{crosshatch}
  \newcommand{\ptrnb}{north west lines}
  \newcommand{\ptrnc}{north east lines}
  \newcommand{\ptrnd}{horizontal lines}

  \newcommand{\ptcla}{red}
  \newcommand{\ptclb}{green}
  \newcommand{\ptclc}{blue}
  \newcommand{\ptcld}{brown}

  \newcommand{\drawrowpart}{
    \draw[pattern=\ptrnd, pattern color=\ptcld] 
    (-\gridwidth - \gridwidth * 0.25, 0 * \gridwidth) 
    rectangle 
    (- \gridwidth * 0.25, 
    0 * \gridwidth + 1 * \gridwidth);
    \draw[pattern=\ptrnc, pattern color=\ptclc] 
    (-\gridwidth - \gridwidth * 0.25, 
    1 * \gridwidth) 
    rectangle 
    (- \gridwidth * 0.25, 
    1 * \gridwidth + 1 * \gridwidth);
    
    \draw[pattern=\ptrnb, pattern color=\ptclb] 
    (-\gridwidth - \gridwidth * 0.25, 
    2 * \gridwidth) 
    rectangle 
    (- \gridwidth * 0.25, 
    2 * \gridwidth + 1 * \gridwidth);
    
    \draw[pattern=\ptrna, pattern color=\ptcla] 
    (-\gridwidth - \gridwidth * 0.25, 
    3 * \gridwidth) 
    rectangle 
    (- \gridwidth * 0.25, 
    3 * \gridwidth + 1 * \gridwidth);
  }
  
  \newcommand{\drawcolpart}[4]{
    \draw[pattern=#3, pattern color=#4] 
    (#1 * \asnwidth, 
    4 * \gridwidth + 0.25 * \gridwidth) 
    rectangle 
    (#2 * \asnwidth + \asnwidth, 
    4 * \gridwidth + 0.25 * \gridwidth + \gridwidth);
  }

%%%%%%%%%%%%%%%%%%%%%%%%%%%%%%%%%%%%%%%%%

\title{DS-MLR: Exploiting Double Separability for Scaling up
  Distributed Multinomial Logistic Regression}
\author{
Parameswaran Raman\\
       {University of California, Santa Cruz}\\
       {params@ucsc.edu}
\and
Sriram Srinivasan\\
        {University of California, Santa Cruz}\\
        {ssriniv9@ucsc.edu}
\and
Shin Matsushima\\
        {University of Tokyo, Japan}\\
        {shin\_matsushima@mist.i.u-tokyo.ac.jp}
\and
Xinhua Zhang\\
        {University of Illinios, Chicago}\\
        {zhangx@uic.edu}        
\and
Hyokun Yun\\
	{Amazon}\\
	{yunhyoku@amazon.com}        
\and
S.V.N Vishwanathan\\
        {University of California, Santa Cruz}\\
        {vishy@ucsc.edu}
}
%\numberofauthors{5}
%\author{
%\alignauthor Hsiang-Fu Yu\\
%       \affaddr{University of Texas, Austin}\\
%       \email{rofuyu@cs.utexas.edu}
%\alignauthor Cho-Jui Hsieh\\
%        \affaddr{University of Texas, Austin}\\
%        \email{cjhsieh@cs.utexas.edu}
%\alignauthor Hyokun Yun\\
%        \affaddr{Amazon}\\
%        \email{yunhyoku@amazon.com}
%\and
%\alignauthor S.V.N Vishwanathan\\
%        \affaddr{University of California, Santa Cruz}\\
%        \email{vishy@ucsc.edu}
%\alignauthor Inderjit S. Dhillon\\
%        \affaddr{University of Texas, Austin}\\
%        \email{inderjit@cs.utexas.edu}
%}

% make the title area
\maketitle

\begin{abstract}
Scaling multinomial logistic regression to datasets with very large number of data points and classes is challenging. This is primarily because one needs to compute the log-partition function on every data point. This makes distributing the computation hard. In this paper, we present a distributed stochastic gradient descent based optimization method (DS-MLR) for scaling up multinomial logistic regression problems to massive scale datasets without hitting any storage constraints on the data and model parameters. Our algorithm exploits double-separability, an attractive property that allows us to achieve both data as well as model parallelism simultaneously. In addition, we introduce a non-blocking and asynchronous variant of our algorithm that avoids bulk-synchronization. We demonstrate the versatility of DS-MLR to various scenarios in data and model parallelism, through an extensive empirical study using several real-world datasets. In particular, we demonstrate the scalability of DS-MLR by solving an extreme multi-class classification problem on the Reddit dataset (159 GB data, 358 GB parameters) where, to the best of our knowledge, no other existing methods apply.
\end{abstract}

%\category{I.2.6}{Artificial Intelligence}{Leraning}
%\terms{Algorithms, Experimentation}
%\keywords{Topic Models; Scalability; Sampling} % NOT required for Proceedings

\section{Introduction}
\label{sec:Introduction}

In this paper, we focus on multinomial logistic regression (MLR) on
massive datasets, in the presence of large number of data points and large number of classes. 
MLR is a method of choice for several real-world tasks such as Image Classification \citep{RusDenSu15} and Video Recommendation \citep{DavLieLiu10}. Therefore, it has received significant research attention \citep{GopYan13}, \citep{YenHuaRavZhoDhi16}. The classic paradigm in distributed machine learning is to perform {\it data partitioning}, using, for instance, a map reduce style architecture where data is distributed across multiple slaves. In each iteration, the slaves gather the parameter vector from the master, compute gradients locally and transmit them back to the master. The L-BFGS optimization algorithm
is typically used in the master to update the parameters after every iteration
\citep{NocWri06}. The main drawback of this strategy is that the model
parameters need to be replicated on every machine. For a $D$ dimensional problem involving $K$ classes, 
this demands $O(K \times D)$ storage. In many cases, this is too large to fit on a single machine.

An orthogonal approach is to use \emph{model partitioning}. Here again, a master slave architecture is used, but now, the data is replicated across each slave. The model parameters are partitioned and
distributed to each machine. During each iteration, the model parameters
on the individual machines are updated, and some auxiliary variables are
computed and distributed to the other slaves, which use these variables
in their parameter updates. See the Log-Concavity (LC) method \cite{GopYan13} for an example of such
a strategy. The main drawback of this approach, however, is that the
data needs to be replicated on each machine, and consequently it is not applicable when the data is too 
large to fit on a single machine.

\begin{wraptable}[13]{r}{.5\linewidth}
\scalebox{0.9}{
    \begin{tabular}{cc|c|c|}
      & \multicolumn{1}{c}{} & \multicolumn{2}{c}{{\it Parameters}}\\
      & \multicolumn{1}{c}{} & \multicolumn{1}{c}{Fit }  & \multicolumn{1}{c}{Do not Fit} \\\cline{3-4}
      \multirow{2}*{{\it Data}}  & Fit & \pbox{7cm}{L-BFGS, LC, \\ {\bf DS-MLR} \\} & LC, {\bf DS-MLR} \\\cline{3-4}
      & Do not Fit & \pbox{7cm}{L-BFGS, \\ {\bf DS-MLR} \\} & {\bf DS-MLR} \\\cline{3-4}
    \end{tabular}}
\caption{Applicability of various methods under different regimes in distributed machine learning. {\bf DS-MLR} is our proposed method. LC is by \cite{GopYan13}.} \label{table:data_model_parallel_compact}
\end{wraptable} 

In contrast to the above approaches, we propose a reformulation of the
objective function of multinomial logistic regression that allows us to
\emph{simultaneously} perform both \emph{data and model partitioning}, thus enabling extreme classification at a massive scale with large number
of data points and classes. We believe this is critical because, the growing acclaim of machine learning is witnessing several novel prediction tasks which not only involve humongous amounts of data, but also strive towards building more sophisticated models, demanding larger storage footprints. Table \ref{table:data_model_parallel_compact} presents a categorization of 
the various methods we discussed. DS-MLR can be applied in all the four scenarios. In Table \ref{table:data_model_parallel_space}, we compare their storage requirements in more detail. DS-MLR occupies the {\it least amount of storage per worker}. The main contributions of this work are:
\begin{enumerate}
	\item We develop DS-MLR, a novel distributed stochastic optimization algorithm that can partition both data as well as model parameters {\it simultaneously} across its workers.
	\item We develop a {\it non-blocking} and {\it asynchronous} variant (DS-MLR Async), which provides further speedups in the multi-core, multi-machine setting by interleaving the computation and communication phases during every iteration.
	\item We present an exhaustive empirical study spanning all the regimes of data and model parallelism, showing that DS-MLR readily applies in all cases. In particular, to demonstrate applicability in the scenario where {\it both data and model do not fit} on a single machine, we run DS-MLR on a large Reddit dataset with data and model parameters occupying {\it 200 GB and 300 GB} respectively.
\end{enumerate}

%\begin{wraptable}[11]{r}{0.6\linewidth}
\begin{table}[h]
\vspace{-1em}
\center
  \renewcommand{\arraystretch}{1.3}
  \scalebox{0.9}{
    \begin{tabular}{l | l c|c}\cline{1-4}
      & \multicolumn{2}{c|}{Storage per worker} & Communication \\ \cline{2-3} 
       &   Data & Parameters \\
      \hline
      L-BFGS &  $O(\frac{N D}{P})$ & $O(K D)$ & $O(K D)$ \\
%      ADMM & 	$O(\frac{N D}{P})$ & $O(K D) + O(\frac{N K}{P})$ &  \\      
      LC &  $O(N D)$ & $O(\frac{K D}{P}) + O(N)$ & $O(N)$ \\
      {\bf DS-MLR} &  $O(\frac{N D}{P})$ & $O(\frac{K D}{P}) + O(\frac{N}{P})$ & $O(\frac{K D}{P})$ \\
    \end{tabular}}
   \caption{Memory requirements of various algorithms when applied to multinomial logistic regression ($N: $ \# of data points, $D: $ \# of features, $K: $ \# of classes, $P: $ \# of workers).}
  \label{table:data_model_parallel_space}    
\end{table}
%\end{wraptable}

%\begin{table}[h]
%\vspace{-1em}
%\center
%  \renewcommand{\arraystretch}{1.3}
%  \scalebox{0.9}{
%    \begin{tabular}{l | l c|c}\cline{1-4}
%      & \multicolumn{2}{c|}{Storage per worker} & Communication \\ \cline{2-3} 
%       &   Data & Parameters \\
%      \hline
%      L-BFGS &  $O(\frac{N D}{P})$ & $O(K D)$ & $O(K D)$ \\
%%      ADMM & 	$O(\frac{N D}{P})$ & $O(K D) + O(\frac{N K}{P})$ &  \\      
%      LC &  $O(N D)$ & $O(\frac{K D}{P}) + O(N)$ & $O(N)$ \\
%      {\bf DS-MLR} &  $O(\frac{N D}{P})$ & $O(\frac{K D}{P}) + O(\frac{N}{P})$ & $O(\frac{K D}{P})$ \\
%    \end{tabular}}
%   \caption{Memory requirements of various algorithms when applied to multinomial logistic regression ($N: $ \# of data points, $D: $ \# of features, $K: $ \# of classes, $P: $ \# of workers).}
%  \label{table:data_model_parallel_space}    
%\end{table}

The rest of the paper is organized as follows: Section \ref{sec:RelatedWork} discusses related work. Section \ref{sec:MultLogistRegr} formally introduces
Multinomial Logistic Regression (MLR). Section \ref{sec:Reformulation} presents our reformulation (DS-MLR). In section \ref{sec:dsmlr_distributed},
we discuss how our doubly separable objective function can be optimized in a 
distributed fashion and present synchronous and asynchronous algorithms for it. In section \ref{sec:convergence}, we present rates of convergence for the synchronous version of DS-MLR. In section \ref{sec:Experiments}, we present empirical results running asynchronous DS-MLR covering all the regimes of data and model parallelism shown in Table \ref{table:data_model_parallel_compact}. Finally, Section \ref{sec:Conclusion} concludes the paper.

%%%%%%%%%%%%%%%%%%%%%%%%%%%%%%%%%%%%%%%%%

\section{Related Work}
\label{sec:RelatedWork}
%In this section, we characterize parallel algorithms for machine learning and discuss related work, thereby putting our DS-MLR method in perspective.\\

There has been a flurry of work in the past few years on developing distributed optimization algorithms for machine learning. Particularly, stochastic gradient descent based approaches have proven to be very fruitful since they make frequent parameter updates and converge much more rapidly \citep{Bot10}. Several algorithms for parallelizing SGD have been proposed in the past such as Hogwild \citep{RecReWriNiu11}, Parallel SGD \citep{ZinWeiLiSmo10}, DSGD \citep{GemNijHaaSis11}, FPSGD \citep{ChiZhuJuaLin13} and more recently, Parameter Server \citep{LiZhoYanLiXia13} and Petuum \citep{xing2015petuum}. Although the importance of data and model parallelism has been recognized in Parameter Server and the Petuum framework \citep{xing2015petuum}, to the best of our knowledge this has not been exploited in their specific instantiations such as applications to multinomial logistic regression \citep{XieKimZhoHoKumYuXin15}. We believe this is because \citep{XieKimZhoHoKumYuXin15} does not reformulate the problem like the way DS-MLR does. Several problems in machine learning are not naturally well-suited for {\it simultaneous} data and model parallelism, and therefore such reformulations are essential in identifying a suitable structure. Moreover, the Parameter Server \citep{LiZhoYanLiXia13} is only a data parallel approach. This is because although the server maintains the model as a distributed key-value store, each worker maintains a working set of the entire model. This is in contrast with DS-MLR where at any given point of time, the model stays truly partitioned into mutually exclusive blocks across the workers.

Alternating direction method of multipliers (ADMM) \citep{BoyParChuPelEck11} is another popular technique used to parallelize convex optimization problems. The key idea in ADMM is to reformulate the original optimization problem 
by introducing redundant linear constraints. This makes the new objective easily {\it data parallel}. However, ADMM suffers from a similar drawback as L-BFGS especially when applied to a multinomial logistic regression model. This is because the number of redundant constraints that need to be introduced are $N$ (\# data points) $\times$ $K$ (\# classes) which is a major bottleneck to model parallelism. Moreover, the convergence rate of ADMM was found to be slow on multinomial logistic regression problems as discussed in \citep{GopYan13}.

Log-Concavity (LC) method \cite{GopYan13} proposed a distributed {\it model parallel} approach to solve the multinomial logistic regression problem by linearizing the log-partition function based on its variational form \citep{Bouchard07}. However, because their formulation is only model parallel, the entire data has to be replicated across all the workers which is not practical for real world applications. Interestingly, we noticed that the objective function of the LC method can also be recovered from (\ref{eq:mlr_obj4}).

Our reformulation in DS-MLR exploits the doubly-separable structure in terms of global model parameters and some local auxiliary variables.
Other doubly-separable methods also exist such as NOMAD \citep{YunYuHsiVisDhi13} for matrix completion and RoBiRank \citep{YunRamVis14} for latent collaborative retrieval. NOMAD \citep{YunYuHsiVisDhi13} is a distributed-memory, asynchronous and decentralized algorithm and RoBiRank \citep{YunRamVis14} is also a distributed-memory but synchronous algorithm. Parameter Server and HogWild \citep{RecReWriNiu11} are asynchronous approaches. In Hogwild, parameter updates are executed in parallel using different threads under the assumption that any two serial updates are not likely to collide on the same data point when the data is sparse. DS-MLR does not make any such assumptions. It has both synchronous and asynchronous variants and the latter is in the spirit of NOMAD.

%Synchronous approaches suffer from non-uniform performance distributions of machines where some machines might happen to be very slow at a given time, thus bringing down the performance of the entire algorithm. Asynchronous methods overcome these drawbacks. 

%In the later sections, we show how our formulation can also be viewed as a type of consensus ADMM method that does not incur these overheads.
%Infact, the $O(N \times K)$ storage is characteristic of most dual-decomposition based methods.

%{\bf (ii) Exact vs. Inexact}: Our work is not the first paper making use of delayed updates for stochastic gradient descent. In \citep{ZinLanSmo09} delayed updates were explored for SGD in the online setting along with theoretical proofs for their convergence. \\

%%%%%%%%%%%%%%%%%%%%%%%%%%%%%%%%%%%%%%%%%

\section{Multinomial Logistic Regression}
\label{sec:MultLogistRegr}
%\subsection{Notation and Preliminaries}
%\label{sec:NotatPrel}
Suppose we are provided training data which consists of $N$ data
points $(\xb_1, y_1), (\xb_2, y_2), \ldots, (\xb_N, y_N)$ where $\xb_i \in
\RR^d$ is a $d$-dimensional feature vector and $y_i \in \cbr{1, 2,
  \ldots, K}$ is a label associated with it; $K$ denotes the number of
class labels. Let's also define an indicator variable $y_{ik} = I(y_i =
k)$ denoting the membership of data point $\xb_i$ to class $k$. 
The probability that $\xb_{i}$ belongs to class $k$ is given by:
\begin{align}
\label{eq:softmax}
p (y = k|\xb_{i}) &= \frac{\exp(\wb_{k}^T \xb_{i})}{\sum_{j=1}^{K} \exp(\wb_j^T \xb_{i})},
\end{align}
where $W = \{\wb_1, \wb_2, \ldots, \wb_K\}$ denotes the parameter vector for each of the $K$ classes.
Using the negative log-likelihood of \eqref{eq:softmax} as a loss function, the objective function of MLR can be written as:
\begin{align}
\label{eq:mlr_obj}
L_1(W) &= \frac{\lambda}{2} \sum_{k=1}^{K} \norm{\wb_k}^2 - \frac{1} {N} \sum_{i=1}^{N} \sum_{k=1}^{K} y_{ik} \wb_k^T \xb_i 
+ \frac{1}{N} \sum_{i=1}^{N} \log \rbr{\sum_{k=1}^{K} \exp(\wb_k^T \xb_i)},
\end{align}
where $\norm{\wb_k}^2$ regularizes the objective, and $\lambda$ is a tradeoff parameter. Optimizing the above objective function (\ref{eq:mlr_obj}) when the
number of classes $K$ is large, is extremely challenging as computing
the \textit{log partition function} involves
summing up over a large number of classes. In addition, it couples the class level parameters $\wb_k$ together,
making it difficult to distribute computation. In this paper, we present an alternative formulation for MLR, to address this challenge.

\section{Doubly-Separable Multinomial Logistic Regression (DS-MLR)}
\label{sec:Reformulation}
In this section, we present a reformulation of the MLR problem, which is closer in spirit to dual-decomposition methods \citep{BoyVan04}.
We begin by first rewriting (\ref{eq:mlr_obj}) as,
\begin{align}
\label{eq:mlr_obj2}
L_1(W) &= \frac{\lambda}{2} \sum_{k=1}^{K} \norm{\wb_k}^2 - \frac{1} {N} \sum_{i=1}^{N} \sum_{k=1}^{K} y_{ik} \wb_k^T \xb_i 
- \frac{1}{N} \sum_{i=1}^{N} \log \frac{1}{\sum_{k=1}^{K} \exp(\wb_k^T \xb_i)},
\end{align}

This can be expressed as a constrained optimization problem,
\begin{align}
\label{eq:mlr_obj3}
L_1(W, A) &= \frac{\lambda}{2} \sum_{k=1}^{K} \norm{\wb_k}^2 - \frac{1} {N} \sum_{i=1}^{N} \sum_{k=1}^{K} y_{ik} \wb_k^T \xb_i 
 - \frac{1}{N} \sum_{i=1}^{N} \log a_{i}, \\
& \text{s.t.} \quad a_{i} = \frac{1}{\sum_{k=1}^{K} \exp(\wb_k^T \xb_i)}, \quad i=1, 2, \ldots N \nonumber
\end{align}
where $A = \{a_i\}_{i=1,\dots,N}$.\\

Observe that this resembles dual-decomposition methods of the form: \\
$\min_{x, z} f(x) + g(z)$ s.t. $Ax + Bz = c$, where $f$ and $g$ are convex functions. In our objective function (\ref{eq:mlr_obj3}), the decomposable functions 
are $f(W)$ and $g(A)$ respectively. Introducing Lagrange multipliers, $\beta_{i}, \quad i=1,2 \ldots N$, we obtain the equivalent unconstrained minimax problem \citep{BoyVan04},
\begin{align}
\label{eq:mlr_obj4}
L_2(W, A, \beta) &= \frac{\lambda}{2} \sum_{k=1}^{K} \norm{\wb_k}^2 - \frac{1} {N} \sum_{i=1}^{N} \sum_{k=1}^{K} y_{ik} \wb_k^T \xb_i 
- \frac{1}{N} \sum_{i=1}^{N} \log a_{i} + \frac{1}{N} \sum_{i=1}^{N} \sum_{k=1}^{K} \beta_{i} \; a_{i} \exp(\wb_k^T \xb_i) 
- \frac{1}{N} \sum_{i=1}^{N} \beta_{i}
\end{align}
It is known that dual-decomposition methods can reliably find a stationary point, therefore the solution obtained by our method is also globally optimal. We discuss the proof of convergence in section \ref{sec:convergence}. The updates for the primal variables $W$, $A$ and dual variable $\beta$ can be written as follows:\\
\begin{align}
\label{eq:update_w}
W_{k}^{t+1} &\leftarrow \argmin_{W_{k}} L_2(W_{k}, a^{t}, \beta^{t}), \\
\label{eq:update_a}
a_{i}^{t+1} &\leftarrow \argmin_{a_{i}} L_2(W_{k}^{t+1}, a_{i}^{t}, \beta_{i}^{t}), \\
\label{eq:update_beta}
\beta_{i}^{t+1} &\leftarrow \beta_{i}^{t} + \rho \rbr{a_{i}^{t+1} \sum_{k=1}^{K} \exp \rbr{{w_{k}^{T}}^{t+1} x_{i}} - 1}
\end{align}
Here, $W_{k}^{t+1}$ and $a_{i}^{t+1}$ can be obtained by any black-box optimization procedure, while $\beta_{i}^{t+1}$ is updated via dual-ascent using a step-length $\rho$. Intuitively, the dual-ascent update of $\beta$ penalizes any violation of the constraint in problem (\ref{eq:mlr_obj3}). 

We now make the following interesting observations in these updates:\\
\textit{Update for} $a_{i}^{t+1}$: When (\ref{eq:update_a}) is solved to optimality, $a_{i}$ admits an exact closed-form solution given by, 
\begin{align}
\label{eq:closedform_a}
a_{i} &= \frac{1}{\beta_{i} \sum_{k=1}^{K} \exp(\wb_k^T \xb_i)},
\end{align}

\textit{Update for} $\beta_{i}^{t+1}$: As a consequence of the above exact solution for $a_{i}$, the dual-ascent update for $\beta_{i}$ is no longer needed, since the penalty is always zero during such a projection if $\beta_{i}$ is set to a constant equal to 1. \\

\textit{Update for} $W_{k}^{t+1}$: This is the only update that we need to handle numerically. \\

$L_2(W, A)$ can be first written in this form,
\begin{align}
\label{eq:mlr_obj5}
 L_2(W, B) &= \sum_{i=1}^{N} \sum_{k=1}^{K} \rbr{ \frac{\lambda}{2N} \norm{\wb_k}^2 - \frac{1} {N} y_{ik} \wb_k^T \xb_i 
- \frac{1}{NK} b_{i} + \frac{1}{N} \exp(\wb_k^T \xb_i + b_{i}) - \frac{1}{N K} }
\end{align}
where we denote $b_i = \log (a_i)$ for convenience and $B = \{b_i\}_{i=1,\dots,N}$. The objective function is now \textit{doubly-separable} \citep{Yun14} since,
\begin{align}
L_2(w_1, \ldots, w_K, b_1, \ldots, b_N) &= \sum\limits_{i=1}^N \sum\limits_{k=1}^K f_{ki}(\wb_{k}, b_i)
\end{align}
where
\begin{align}
f_{ki}(\wb_{k}, b_i) &= \frac{\lambda}{2N} \|\wb_{k}\|^2 - \frac{y_{ik} \wb_{k}^T\xb_{i}}{N}  + \frac{1}{N} \exp(\wb_{k}^T\xb_{i} + b_i) 
- \frac{b_i}{N K} - \frac{1}{N K} .
\end{align}
Obtaining such a form for the objective function is key to achieving simultaneous data and model parallelism. It is worth pointing out that such an objective function can also be derived using the variational form for the log-partition function \citep{Bouchard07}.

\textit{Stochastic Optimization:}
%(BUGBUG: stochastic updates involve \beta)
Minimizing $L_2(W, B)$ involves computing the gradients of
eqn (\ref{eq:mlr_obj5}) w.r.t. $\wb_{k}$ which is
often computationally expensive. Instead, one can compute
\emph{stochastic gradients} \citep{RobMon51} which are computationally
cheaper than the exact gradient, and perform stochastic updates as
follows:
\begin{align}
&\wb_{k} \leftarrow \wb_{k} - \eta K\rbr{\lambda \wb_{k} - y_{ik} \xb_{i} + \exp(\wb_{k}^T\xb_{i} + b_i) \xb_{i}} \label{eq:dsmlr_sgdupdatew}
%&b_i \leftarrow b_i - \eta_2 K\rbr{ \exp(\wb_{k}^T\xb_{i} + b_i) - \frac{1}{K}} \label{eq:dsmlr_sgdupdateb}
\end{align}
where $\eta$ is the learning rate for $\wb_{k}$. Being an unbiased stochastic gradient estimator, the standard convergence guarantees of SGD apply here \citep{KusYin03}.
%Although at first glance, our formulation of DS-MLR (\ref{eq:dsmlr_obj}) looks similar to objective (\ref{eq:mlr_newobj}) of \cite{GopYan13},  there are some key advantages of reformulating the objective in this manner: 

Our formulation of DS-MLR in eqn (\ref{eq:mlr_obj5}) offers several key advantages: 
\begin{enumerate}
\item Observe that the objective function $L_2(W, B)$, now splits as summations over $N$ data points and $K$ classes. This means, each term in stochastic updates only depends one data point $i$ and one class $k$. We exploit this to achieve simultaneous data and model parallelism.
\item We are able to update the variational parameters $b_i$ in closed-form, avoiding noisy stochastic updates. This improves our overall convergence. 
\item Our formulation lends itself nicely to an asynchronous implementation. Section \ref{sec:dsmlr_async} describes this in more detail.
\item Traditionally, dual-decomposition methods (e.g. ADMM) have been able to exploit distributed computing by solving separable sub-problems on multiple machines. Although this naturally led to {\it data parallelism}, it has not been clear so far how to {\it separate the model parameters}. For e.g. when applied to the MLR problem in the naive way, these methods introduce an additional $O\rbr{N K}$ storage (please refer to discussion in section 3.4 in \citep{GopYan13}). In our reformulation (which bears a flavor of dual-decomposition methods), we demonstrate one way in which we can achieve both of these goals simultaneously. This is done by constraining the problem differently as shown in eqn (\ref{eq:mlr_obj3}) in our paper.
\end{enumerate}

%Our formulation can also be viewed as one way of doing consensus ADMM \citep{ZhaKwo14}. Intuitively, the $a_{i}$ variables can be interpreted in the form of consensus variables which capture the combined effect of the parameters $w_{\kb}$ \; $i = 1 \ldots K$. \\
%(BUGBUG: Need to either make this more formal or just say our updates look the same as that of consensus ADMM).

%%%%%%%%%%%%%%%%%%%%%%%%%%%%%%%%%%%%%%%%%
\section{Distributing the Computation of DS-MLR}
\label{sec:dsmlr_distributed}

\subsection{DS-MLR Sync}
\label{sec:dsmlr_sync}
We first describe the distributed DS-MLR Synchronous algorithm in Algorithm~\ref{alg:dsmlr_1delay}. The data and parameters are distributed among the $P$ processors as illustrated in Figure \ref{fig:ds_sgd} where the row-blocks and column-blocks represent data $X^{(p)}$ and weights $W^{(p)}$ on each local processor respectively. The algorithm proceeds by running $T$ iterations in parallel on each of the $P$ workers arranged in a ring network topology. 
%Each iteration consist of $P$ inner-epochs. During the inner-epoch, each worker first exchanges its parameters $W^{(p)}$ with the adjacent machines. Next, it updates the block of weight parameters $W^{(p)}$ and variational parameters $b^{(p)}$ that it owns. Observe that these updates depend only on $X^{(p)}$ and $W^{(p)}$.

%\begin{algorithm}[!ht]
%  \begin{algorithmic}[1]
%    \STATE{$K$: \# classes, $P$: \# workers, $T$: total outer iterations, $t$: outer iteration index, $s$: inner epoch index}
%    \STATE{$W^{(p)}$: weights per worker, $b^{(p)}$: variational parameters per worker}
%    %\FORALL {$p = 1, 2, \ldots P \quad \textbf{\text{in parallel}}$}
%        \STATE{Initialize $W^{(p)}=0$, $b^{(p)}=\frac{1}{K}$}
%        \FORALL { $p = 1, 2, \ldots, P $ in parallel}
%	\FORALL {$t = 1, 2, \ldots, T$}
%		\FORALL{$s = 1, 2, \ldots, P$}
%			\STATE{\text{Send} $W^{(p)}$ \text{to worker on the right}}
%			\STATE{\text{Receive} $W^{(p)}$ \text{from worker on the left}}
%			\STATE{\text{Update} $W^{(p)}$ \text{stochastically using (\ref{eq:dsmlr_sgdupdatew})}}
%			\STATE{\text{Update} $b^{(p)}$ \text{stochastically using (\ref{eq:dsmlr_sgdupdateb})}}
%		\ENDFOR
%	\ENDFOR
%	\ENDFOR
%    %\ENDFOR
%  \end{algorithmic}
%  \caption{DS-MLR Synchronous}
%  \label{alg:dsmlr}
%\end{algorithm}

\begin{figure}[ht]
	\centering
	\includegraphics[width=5in]{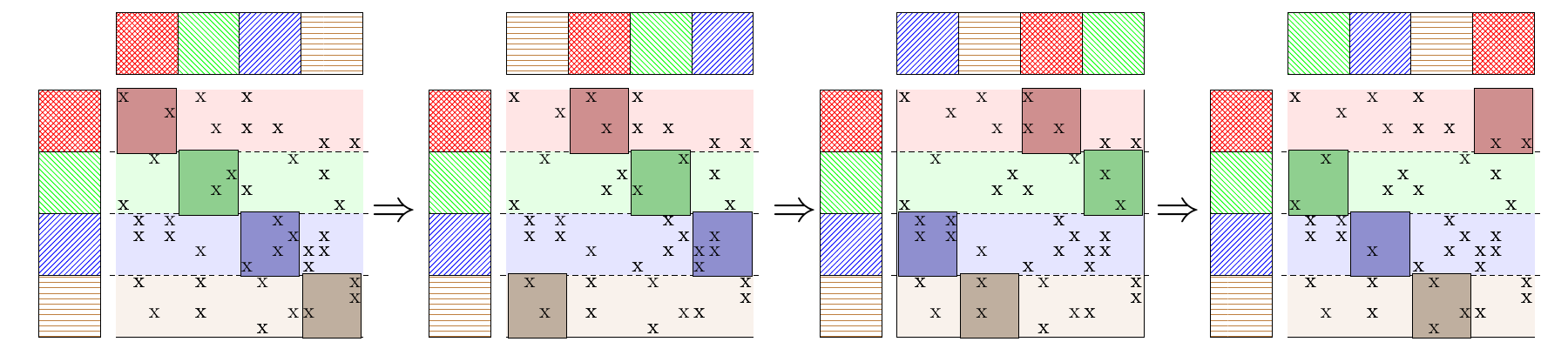}
	\caption{$P=4$ inner-epochs of distributed SGD. Each worker updates mutually-exclusive blocks of data and parameters as shown by the dark colored diagonal blocks \citep{GemNijHaaSis11}.}
	\label{fig:ds_sgd}
\end{figure}

%At this point, it is worth noting that this is an instantiation of the DSGD (Distributed Stochastic Gradient Descent) scheme in \cite{gemulla2011large} which proves the scheme's asymptotic convergence properties. This was actively extended by \cite{YunYuHsiVisDhi13} and \cite{ChiZhuJuaLin13}. To our knowledge, however, DSGD and its extensions have only been used for matrix factorization problems where double separability can be immediately seen, and we are the first to apply the scheme beyond matrix factorization.

Each iteration consist of $2P$ inner-epochs. During the first $P$ inner-epochs, each worker sends/receives its parameters $W^{(p)}$ to/from the adjacent machine and performs stochastic $W^{(p)}$ updates using the block of data $X^{(p)}$ and parameters $W^{(p)}$ that it owns. The second $P$ inner-epochs are used to pass around the $W^{(p)}$ to compute the $b^{(p)}$ exactly using (\ref{eq:closedform_a}). 
%In this approach, the values of $\wb_{k}$'s used to update $b_{i}$ are stale by one round of inner-epochs.

\begin{algorithm}[!ht]
  \begin{algorithmic}[1]
    \STATE {$K$: \# classes, $P$: \# workers, $T$: total outer iterations, $t$: outer iteration index, $s$: inner epoch index}
    \STATE{$W^{(p)}$: weights per worker, $b^{(p)}$: variational parameters per worker}

    %\FORALL {$p = 1, 2, \ldots P \quad \textbf{\text{in parallel}}$}
        \STATE{Initialize $W^{(p)}=0$, $b^{(p)}=\frac{1}{K}$}
        \FORALL { $p = 1, 2, \ldots, P $ in parallel}
	\FORALL {$t = 1, 2, \ldots, T$}
		\FORALL{$s = 1, 2, \ldots, P$}
			\STATE{\text{Send} $W^{(p)}$ \text{to worker on the right}}
			\STATE{\text{Receive} $W^{(p)}$ \text{from worker on the left}}
			\STATE{\text{Update} $W^{(p)}$ \text{stochastically using (\ref{eq:dsmlr_sgdupdatew})}}
		\ENDFOR
		
		\FORALL{$s = 1, 2, \ldots, P$}
			\STATE{\text{Send} $W^{(p)}$ \text{to worker on the right}}
			\STATE{\text{Receive} $W^{(p)}$ \text{from worker on the left}}
                         \STATE{\text{Compute partial sums}}
		\ENDFOR
		\STATE{\text{Update} $b^{(p)}$} \text{exactly (\ref{eq:closedform_a}) using the partial sums}
	\ENDFOR
	\ENDFOR
    %\ENDFOR
  \end{algorithmic}
  \caption{DS-MLR Synchronous}
  \label{alg:dsmlr_1delay}
\end{algorithm}

%{\bf Avoiding Bulk Synchronization:}
%One could also update the $b_{i}$ in the background using the closed form (\ref{eq:closedform_a}) simultaneously while updating the $\wb_{k}$'s. This would ensure the freshest copy of $\wb_{k}$ is utilized in updating the variational parameters, and avoid a separate bulk synchronization step. This version of the 
%synchronous algorithm is outlined in Algorithm~\ref{alg:dsmlr_0delay} in the Appendix.

%The 1-delay algorithm uses a synchronization step after the first $P$ inner-epochs where the machines stop, exchange the parameters and update $b_i$ exactly which is expensive. In order to get the best of both worlds - i.e., make use of closed form updates as well as avoid the additional synchronization step, we devise another variant of our method which we term as {\em 0-delay algorithm}. The main idea in 0-delay is to update $b_i$ in the background using their closed form (\ref{eq:closedform_a}) simultaneously while updating the $\wb_{k}$'s. Since this approach uses the freshest copy of $\wb_{k}$, we term it \emph{0-delay}. Computing $b_i$ requires computing the partial sums $\sum_{k=1}^K \exp(\wb_{k}^T \xb_{i})$, so we compute these partial sums in the inner-epochs and then use them at the end to make the final update. This completely avoids any bulk synchronization in the outer-epochs. This algorithm is outlined in Algorithm~\ref{alg:dsmlr_0delay} in the Appendix.

\subsection{DS-MLR Async}
\label{sec:dsmlr_async}
The performance of DS-MLR can be significantly improved by performing computation and communication in parallel. Based on this observation, we present an asynchronous version of DS-MLR. Due to the double-separable nature of our objective function (\ref{eq:mlr_obj5}), we can readily apply the NOMAD algorithm proposed in \cite{YunYuHsiVisDhi13}. The entire DS-MLR Async algorithm is described in Algorithm \ref{alg:dsmlr_nomad}.

\begin{algorithm}[!ht]
\begin{algorithmic}[1]
	\STATE {$K$: total \# classes, $P$: total \# workers, $T$: total outer iterations, $W^{(p)}$: weights per worker}
	\STATE{$b^{(p)}$: variational parameters per worker, \; queue[$P$]: array of $P$ worker queues}
	\vspace*{0.1cm}
	\STATE{Initialize $W^{(p)}=0$, $b^{(p)}=\frac{1}{K}$ {\small \texttt{//Initialize parameters}}}
	\vspace*{0.1cm}
	\FOR {$k \in W^{(p)}$}
        		\STATE{Pick $q$ uniformly at random}
%        		\STATE{Compute index of queue to push: $q = k \pmod P$}		
        		\STATE{queue[$q$].push($(k, \wb_{k}))$ {\small \texttt{//Initialize worker queues}}}
        \ENDFOR
        	\vspace*{0.1cm}
	\STATE{{\small \texttt{//Start P workers}}}        
        \FORALL { $p = 1, 2, \ldots, P $ in parallel}
        		\FORALL {$t = 1, 2, \ldots, T$}
			\REPEAT
				\STATE {$(k, \wb_{k}) \leftarrow \text{queue[p].pop()}$}
				\STATE {\text{Update} $\wb_{k}$ \text{stochastically using (\ref{eq:dsmlr_sgdupdatew})}}
				\STATE{\text{Compute partial sums}}
				\STATE{Compute index of next queue to push to: $\qhat$}
				\STATE{queue[$\qhat$].push($(k, \wb_{k}))$}
			\UNTIL{\# of updates is equal to $K$}
			\STATE{\text{Update} $b^{(p)}$} \text{exactly (\ref{eq:closedform_a}) using the partial sums}
		\ENDFOR
        \ENDFOR
  \end{algorithmic}
  \caption{DS-MLR Asynchronous}
  \label{alg:dsmlr_nomad}
\end{algorithm}

The algorithm begins by distributing the data and parameters among $P$ workers in the same fashion as in the 
synchronous version. However, here we also maintain $P$ worker queues. Initially the parameters $W^{(p)}$ are 
distributed uniformly at random across the queues. The workers subsequently
can run their updates in parallel as follows: each one pops a parameter $\wb_{k}$ out the queue, updates it stochastically and pushes it into the queue of the next worker. Simultaneously, each worker also records the partial sum (the local contribution of each worker towards the global normalization constant $\sum_{k=1}^K exp(w_k^T x_i)$) that is required for updating the variational parameters. This process repeats until $K$ updates have been made which is equivalent to saying that each worker has updated every parameter $\wb_{k}$. Following this, the worker updates all its variational parameters $b^{(p)}$ exactly using the partial sums (\ref{eq:closedform_a}). For simplicity of explanation, we restricted Algorithm \ref{alg:dsmlr_nomad} to $P$ workers on a single-machine. However, in our actual implementation, we follow a Hybrid Architecture. This means that there are multiple threads running on a single machine in addition to multiple machines sharing the load across the network. Therefore, in this setting, each worker (thread) first passes around the parameter $\wb_{k}$ across all the threads on its machine. Once this is completed, the parameter is tossed onto the queue of the first thread on the next machine. Such a Hybrid Architecture does improve the experimental performance significantly since communication among threads locally within a machine is much less expensive than communication across threads which reside on different machines in the network. 

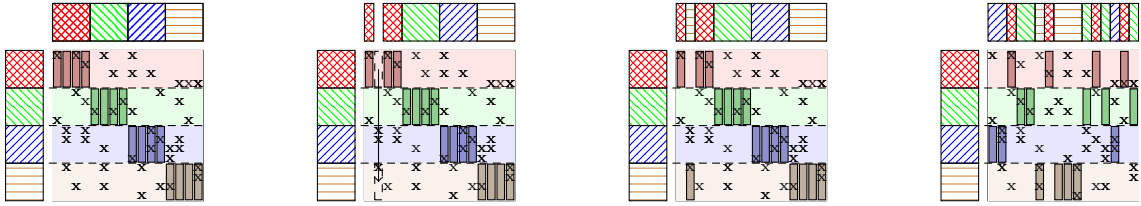
\begin{figure}[ht]
  \begin{subfigure}[t]{0.22\textwidth}
%    \centering
    \begin{tikzpicture}[scale=0.25]
      
      \draw (0,0) rectangle (\gridwidth * \nodenum, \gridwidth * \nodenum);
  
    \fill [red!10] (0, 3 * \gridwidth) 
    rectangle (\gridwidth * \nodenum, 4 * \gridwidth);

    \fill [green!10] (0, 2 * \gridwidth) 
    rectangle (\gridwidth * \nodenum, 3 * \gridwidth);

    \fill [blue!10] (0, 1 * \gridwidth) 
    rectangle (\gridwidth * \nodenum, 2 * \gridwidth);

    \fill [brown!10] (0,0) 
    rectangle (\gridwidth * \nodenum, 1 * \gridwidth);

    % split between nodes
    \foreach \y in {1,2,3}
    {
      \draw[densely dashed] 
      (-\gridwidth * 0.1, \gridwidth * \y)
      --
      (\gridwidth * \nodenum + \gridwidth * 0.1,  \gridwidth * \y);
    }
    
    \drawrowpart

    \drawcolpart{0}{3}{\ptrna}{\ptcla}
    \drawcolpart{4}{7}{\ptrnb}{\ptclb}
    \drawcolpart{8}{11}{\ptrnc}{\ptclc}
    \drawcolpart{12}{15}{\ptrnd}{\ptcld}

    \dtpt{5}{1} \dtpt{0}{15} \dtpt{13}{12} \dtpt{13}{6} \dtpt{14}{8}
    \dtpt{0}{8} \dtpt{7}{10} \dtpt{15}{2} \dtpt{10}{5} \dtpt{8}{13}
    \dtpt{5}{3} \dtpt{12}{5} \dtpt{10}{7} \dtpt{11}{6} \dtpt{15}{3}
    \dtpt{6}{9} \dtpt{13}{5} \dtpt{1}{3} \dtpt{8}{15} \dtpt{1}{6}
    \dtpt{1}{7} \dtpt{8}{9} \dtpt{12}{4} \dtpt{12}{1} \dtpt{13}{10}
    \dtpt{9}{0} \dtpt{8}{4} \dtpt{15}{12} \dtpt{3}{6} \dtpt{10}{13}
    \dtpt{6}{13} \dtpt{2}{11} \dtpt{5}{5} \dtpt{2}{1} \dtpt{3}{7}
    \dtpt{9}{3} \dtpt{11}{1} \dtpt{3}{14} \dtpt{11}{11} \dtpt{5}{15}

    \drawcolpart{0}{3}{\ptrna}{\ptcla}
    \drawcolpart{4}{7}{\ptrnb}{\ptclb}
    \drawcolpart{8}{11}{\ptrnc}{\ptclc}
    \drawcolpart{12}{15}{\ptrnd}{\ptcld}

    \asnbox{0}{0}{red} \asnbox{1}{0}{red} \asnbox{2}{0}{red} \asnbox{3}{0}{red}
    \asnbox{4}{1}{green} \asnbox{5}{1}{green} \asnbox{6}{1}{green} \asnbox{7}{1}{green}
    \asnbox{8}{2}{blue} \asnbox{9}{2}{blue} \asnbox{10}{2}{blue} \asnbox{11}{2}{blue}
    \asnbox{12}{3}{brown} \asnbox{13}{3}{brown} \asnbox{14}{3}{brown} \asnbox{15}{3}{brown}
      
    \dtpt{5}{1} \dtpt{0}{15} \dtpt{13}{12} \dtpt{13}{6} \dtpt{14}{8}
    \dtpt{0}{8} \dtpt{7}{10} \dtpt{15}{2} \dtpt{10}{5} \dtpt{8}{13}
    \dtpt{5}{3} \dtpt{12}{5} \dtpt{10}{7} \dtpt{11}{6} \dtpt{15}{3}
    \dtpt{6}{9} \dtpt{13}{5} \dtpt{1}{3} \dtpt{8}{15} \dtpt{1}{6}
    \dtpt{1}{7} \dtpt{8}{9} \dtpt{12}{4} \dtpt{12}{1} \dtpt{13}{10}
    \dtpt{9}{0} \dtpt{8}{4} \dtpt{15}{12} \dtpt{3}{6} \dtpt{10}{13}
    \dtpt{6}{13} \dtpt{2}{11} \dtpt{5}{5} \dtpt{2}{1} \dtpt{3}{7}
    \dtpt{9}{3} \dtpt{11}{1} \dtpt{3}{14} \dtpt{11}{11} \dtpt{5}{15}

    \dtpt{2}{15} \dtpt{3}{10} \dtpt{4}{9} \dtpt{15}{12} \dtpt{14}{12}

    \end{tikzpicture}
    \caption{Initial assignment of $W$ and $X$. Each worker works
      only on the diagonal active area in the beginning.}
  \end{subfigure}
  \quad
  \begin{subfigure}[t]{0.22\textwidth}
%    \centering
    \begin{tikzpicture}[scale=0.25]

      \boxlayout
      
      \drawrowpart

      \drawcolpart{0}{0}{\ptrna}{\ptcla}
      \drawcolpart{2}{3}{\ptrna}{\ptcla}
      \drawcolpart{4}{7}{\ptrnb}{\ptclb}
      \drawcolpart{8}{11}{\ptrnc}{\ptclc}
      \drawcolpart{12}{15}{\ptrnd}{\ptcld}

      \asnbox{0}{0}{red} \asnbox{2}{0}{red} \asnbox{3}{0}{red}
      \asnbox{4}{1}{green} \asnbox{5}{1}{green} \asnbox{6}{1}{green} \asnbox{7}{1}{green}
      \asnbox{8}{2}{blue} \asnbox{9}{2}{blue} \asnbox{10}{2}{blue} \asnbox{11}{2}{blue}
      \asnbox{12}{3}{brown} \asnbox{13}{3}{brown} \asnbox{14}{3}{brown} \asnbox{15}{3}{brown}
      
      \dasnbox{1}{0} \dasnbox{1}{3}
      \mvarr{1}{0}{1}{3}

      \dtpt{5}{1} \dtpt{0}{15} \dtpt{13}{12} \dtpt{13}{6} \dtpt{14}{8}
      \dtpt{0}{8} \dtpt{7}{10} \dtpt{15}{2} \dtpt{10}{5} \dtpt{8}{13}
      \dtpt{5}{3} \dtpt{12}{5} \dtpt{10}{7} \dtpt{11}{6} \dtpt{15}{3}
      \dtpt{6}{9} \dtpt{13}{5} \dtpt{1}{3} \dtpt{8}{15} \dtpt{1}{6}
      \dtpt{1}{7} \dtpt{8}{9} \dtpt{12}{4} \dtpt{12}{1} \dtpt{13}{10}
      \dtpt{9}{0} \dtpt{8}{4} \dtpt{15}{12} \dtpt{3}{6} \dtpt{10}{13}
      \dtpt{6}{13} \dtpt{2}{11} \dtpt{5}{5} \dtpt{2}{1} \dtpt{3}{7}
      \dtpt{9}{3} \dtpt{11}{1} \dtpt{3}{14} \dtpt{11}{11} \dtpt{5}{15}
      \dtpt{2}{15} \dtpt{3}{10} \dtpt{4}{9} \dtpt{15}{12} \dtpt{14}{12}

    \end{tikzpicture}
    \caption{After a worker finishes processing column $k$, it sends
      the corresponding item parameter $\wb_{k}$ to another worker.
      Here, $\wb_{2}$ is sent from worker $1$ to $4$.  }
  \end{subfigure}
  \quad
  \begin{subfigure}[t]{0.22\textwidth}
%    \centering
    \begin{tikzpicture}[scale=0.25]

      \boxlayout
    
      \drawrowpart

      \drawcolpart{0}{0}{\ptrna}{\ptcla}
      \drawcolpart{1}{1}{\ptrnd}{\ptcld}
      \drawcolpart{2}{3}{\ptrna}{\ptcla}
      \drawcolpart{4}{7}{\ptrnb}{\ptclb}
      \drawcolpart{8}{11}{\ptrnc}{\ptclc}
      \drawcolpart{12}{15}{\ptrnd}{\ptcld}

      \asnbox{0}{0}{red} \asnbox{2}{0}{red} \asnbox{3}{0}{red}
      \asnbox{4}{1}{green} \asnbox{5}{1}{green} \asnbox{6}{1}{green} \asnbox{7}{1}{green}
      \asnbox{8}{2}{blue} \asnbox{9}{2}{blue} \asnbox{10}{2}{blue} \asnbox{11}{2}{blue}
      \asnbox{12}{3}{brown} \asnbox{13}{3}{brown} \asnbox{14}{3}{brown} \asnbox{15}{3}{brown}
      \asnbox{1}{3}{brown}

      \dtpt{5}{1} \dtpt{0}{15} \dtpt{13}{12} \dtpt{13}{6} \dtpt{14}{8}
      \dtpt{0}{8} \dtpt{7}{10} \dtpt{15}{2} \dtpt{10}{5} \dtpt{8}{13}
      \dtpt{5}{3} \dtpt{12}{5} \dtpt{10}{7} \dtpt{11}{6} \dtpt{15}{3}
      \dtpt{6}{9} \dtpt{13}{5} \dtpt{1}{3} \dtpt{8}{15} \dtpt{1}{6}
      \dtpt{1}{7} \dtpt{8}{9} \dtpt{12}{4} \dtpt{12}{1} \dtpt{13}{10}
      \dtpt{9}{0} \dtpt{8}{4} \dtpt{15}{12} \dtpt{3}{6} \dtpt{10}{13}
      \dtpt{6}{13} \dtpt{2}{11} \dtpt{5}{5} \dtpt{2}{1} \dtpt{3}{7}
      \dtpt{9}{3} \dtpt{11}{1} \dtpt{3}{14} \dtpt{11}{11} \dtpt{5}{15}
      \dtpt{2}{15} \dtpt{3}{10} \dtpt{4}{9} \dtpt{15}{12}
      \dtpt{14}{12}

    \end{tikzpicture}
    \caption{Upon receipt, the column is processed by the new worker.
      Here, worker $4$ can now process column $2$ since it owns the
      column.}
  \end{subfigure}
  \quad
  \begin{subfigure}[t]{0.22\textwidth}
    \begin{tikzpicture}[scale=0.25]

      \boxlayout

      \drawrowpart

      \asnbox{0}{2}{blue} \asnbox{1}{2}{blue} \asnbox{2}{0}{red} \asnbox{3}{1}{green}
      \asnbox{4}{1}{green} \asnbox{5}{3}{brown} \asnbox{6}{0}{red} \asnbox{7}{3}{brown}
      \asnbox{8}{3}{brown} \asnbox{9}{3}{brown} \asnbox{10}{1}{green} \asnbox{11}{0}{red}
      \asnbox{12}{1}{green} \asnbox{13}{2}{blue} \asnbox{14}{0}{red} \asnbox{15}{1}{green}

      \drawcolpart{0}{1}{\ptrnc}{\ptclc}
      \drawcolpart{2}{2}{\ptrna}{\ptcla}
      \drawcolpart{3}{4}{\ptrnb}{\ptclb}
      \drawcolpart{5}{5}{\ptrnd}{\ptcld}
      \drawcolpart{6}{6}{\ptrna}{\ptcla}
      \drawcolpart{7}{9}{\ptrnd}{\ptcld}
      \drawcolpart{10}{10}{\ptrnb}{\ptclb}
      \drawcolpart{11}{11}{\ptrna}{\ptcla}
      \drawcolpart{12}{12}{\ptrnb}{\ptclb}
      \drawcolpart{13}{13}{\ptrnc}{\ptclc}
      \drawcolpart{14}{14}{\ptrna}{\ptcla}
      \drawcolpart{15}{15}{\ptrnb}{\ptclb}
      
      \dtpt{5}{1} \dtpt{0}{15} \dtpt{13}{12} \dtpt{13}{6} \dtpt{14}{8}
      \dtpt{0}{8} \dtpt{7}{10} \dtpt{15}{2} \dtpt{10}{5} \dtpt{8}{13}
      \dtpt{5}{3} \dtpt{12}{5} \dtpt{10}{7} \dtpt{11}{6} \dtpt{15}{3}
      \dtpt{6}{9} \dtpt{13}{5} \dtpt{1}{3} \dtpt{8}{15} \dtpt{1}{6}
      \dtpt{1}{7} \dtpt{8}{9} \dtpt{12}{4} \dtpt{12}{1} \dtpt{13}{10}
      \dtpt{9}{0} \dtpt{8}{4} \dtpt{15}{12} \dtpt{3}{6} \dtpt{10}{13}
      \dtpt{6}{13} \dtpt{2}{11} \dtpt{5}{5} \dtpt{2}{1} \dtpt{3}{7}
      \dtpt{9}{3} \dtpt{11}{1} \dtpt{3}{14} \dtpt{11}{11} \dtpt{5}{15}

      \dtpt{2}{15} \dtpt{3}{10} \dtpt{4}{9} \dtpt{15}{12} \dtpt{14}{12}

    \end{tikzpicture}
    \caption{During the execution of the algorithm, the ownership of the
      global parameters (weight vectors) $\wb_{k}$ changes.}
  \end{subfigure}
  \caption{Illustration of the communication pattern in DS-MLR Async algorithm (based on the NOMAD algorithm \cite{YunYuHsiVisDhi13})}
  \label{fig:nomad_scheme}
\end{figure}

%%%%%%%%%%%%%%%%%%%%%%%%%%%%%%%%%%%%%%%

\section{Convergence}
\label{sec:convergence}
%(BUGBUG: Revisit this)
Although the semi-stochastic nature of DS-MLR makes it hard to directly apply the existing convergence results, under standard assumptions, it can be shown that it finds $\epsilon$ accurate solutions to the original objective $L_1$ in $T=O(1/\epsilon^2)$ iterations.
\begin{theorem}
\label{thrm:syncdsmlr_proof}
  Suppose all $\nbr{\xb_{i}} \le r$ for a constant $r > 0$.
  Let the step size $\eta$ in \eqref{eq:dsmlr_sgdupdatew} decay at the rate of $1/\sqrt{t}$.
  Then, $\exists$ constant $C$ independent of $N, K, D$ and $P$,
  such that
  \begin{align}
    \min_{t=1,\ldots,T} L_1(W^t) - L_1(W) \le \frac{C}{\sqrt{t}},\quad \forall W,
  \end{align}
  where $W^t$ is value of $W$ at the end of the iteration $t$ and $\xb_{i}$ denotes the data point. N, K, D, and P denote the number of data points, classes, dimensions and workers respectively.
\end{theorem}
It is worth noting that this rate of convergence is independent of the size of the problem.
In particular, it is invariant to $P$, the number of workers. Therefore, as more workers become available,
the computational cost per iteration can be effectively distributed without sacrificing the overall convergence rate,
up to the point where communication cost becomes dominant. Detailed proof is relegated to the Appendix \ref{sec:app_rate}. 

Our key idea in casting both algorithms as stochastic gradient descent methods is to demonstrate that although the update of $W$ is based on a stale value of $b$ arising from the delayed updates,
such a delay still allows the error of the gradient of $L_1$ w.r.t $W$ to be bounded by $O(\eta)$, in Euclidean norm.
%As a side note, a linear rate of convergence can also be obtained for DS-MLR by following proofs on the lines of SVRG \citep{JohZha}.

It should be noted that at this point our analysis and Theorem \ref{thrm:syncdsmlr_proof} applies only to the synchronous version of DS-MLR. The asynchronous version can be analyzed by
upper-bounding the delay parameter and following proof techniques on the lines of \citep{LebPedLac16} and \citep{LiuWriReBitSri15}.

%%%%%%%%%%%%%%%%%%%%%%%%%%%%%%%%%%%%%%%%%

\section{Experiments}
\label{sec:Experiments}
In our empirical study, we will focus on DS-MLR Async. We use a wide scale of real-world datasets of varying characteristics which is described in Table \ref{table:datasets}. Our experimental setup follows the same categorization we outline in Table \ref{table:data_model_parallel_compact}. 
%LC method has already been shown to perform better in terms of wallclock time to convergence when compared to ADMM \citep{GopYan13}, so we omit comparing ADMM in our experiments.

\begin{table*}[!htbp]
\begin{flushleft}
  \renewcommand{\arraystretch}{1}
  \scalebox{0.75}{
\begin{tabular}{c|c|c|c|c|c|c|c}
Dataset & \# instances & \# features & \#classes & data (train + test) & parameters & sparsity (\% nnz) & Methods that apply \\ \hline
CLEF & 10,000 & 80 & 63 & 9.6 MB + 988 KB & 40 KB & 100 & {\it L-BFGS, LC, {\bf DS-MLR}} \\ %\hline
NEWS20 & 11,260 & 53,975 & 20 & 21 MB + 14 MB & 9.79 MB & 0.21 & {\it L-BFGS, LC, {\bf DS-MLR}} \\ %\hline
LSHTC1-small & 4,463 & 51,033 & 1,139 & 11 MB + 4 MB & 465 MB & 0.29 & {\it L-BFGS, LC, {\bf DS-MLR}} \\ %\hline
LSHTC1-large & 93,805 & 347,256 & 12,294 & 258 MB + 98 MB & 34 GB & 0.049 & {\it LC, {\bf DS-MLR}} \\ %\hline
%Wikipedia-Large & 2,365,436 & 20,000 & 325,056 & 965 MB + 388 MB & 52 GB & 0.14 \\ %\hline
ODP & 1,084,404 & 422,712 & 105,034 & 3.8 GB + 1.8 GB & 355 GB & 0.0533 & {\it LC, {\bf DS-MLR}} \\
YouTube8M-Video & 4,902,565 & 1,152 & 4,716 &  59 GB + 17 GB & 43 MB & 100 & {\it L-BFGS, {\bf DS-MLR}} \\
%YouTube-Frame & 1,404,423 & 138,240 & 4716 & 2.5 TB + 618 GB & 5.2 GB & 0 & \\
Reddit-Small & 52,883,089 & 1,348,182 & 33,225 &  40 GB + 18 GB & 358 GB & 0.0036 & {\it {\bf DS-MLR}} \\
Reddit-Full & 211,532,359 & 1,348,182 & 33,225 & 159 GB + 69 GB & 358 GB & 0.0036 & {\it {\bf DS-MLR}} \\
%\hline
\end{tabular}}
\end{flushleft}
\caption{Dataset Characteristics}
\label{table:datasets}
\end{table*}

{\bf Hardware:} All single-machine experiments were run on a cluster with the configuration of two 8-core Intel Xeon-E5 processors and 32 GB memory per node. For multi-machine multi-core, we used Intel vLab Knights Landing (KNL) cluster with node configuration of Intel Xeon Phi 7250 CPU (64 cores, 200GB memory), connected through Intel Omni-Path (OPA) Fabric. The asynchronous, non-blocking property of DS-MLR makes it ideal to be run on KNL, which is a many-core (68 core, 272 threads) architecture with massive FPLOPs, memory bandwidth, and large memory space (MCDRAM + DDR).

{\bf Implementation Details:} We implemented our DS-MLR method in C++ using MPI for communication across nodes and Intel TBB for concurrent queues and multi-threading. To make the comparison fair, we re-implemented the LC \cite{GopYan13} method in C++ and MPI using ALGLIB for the inner optimization. Finally, for the L-BFGS baseline, we used the TAO solver (from PETSc).

\subsection{Data Fits and Model Fits}
For this experiment, we compare DS-MLR, L-BFGS and the LC methods on small scale datasets CLEF, NEWS20, LSHTC1-small. 

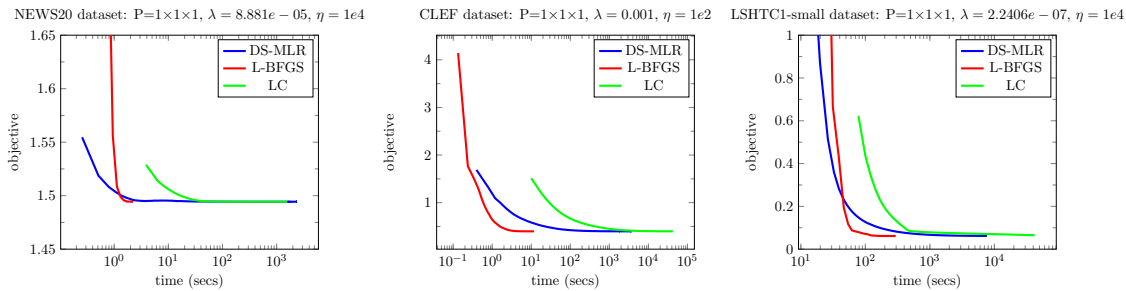
\begin{figure}[H]
 \centering
 \begin{tikzpicture}[scale=0.5]
    \begin{axis}[xmode=log, minor tick num=1,
      title={NEWS20 dataset: P=1$\times$1$\times$1, $\lambda=8.881e-05$, $\eta=1e4$},
      xlabel={time (secs)}, ylabel={objective}, ymin={1.45}, ymax={1.65}]

      \addplot[ultra thick, color=blue] table [x index=2, y index=1, header=true]
      {../../Experiments/serial/logs/dsmlrnomad_NEWS20.progress};
      \addlegendentry{DS-MLR}

      \addplot[ultra thick, color=red] table [x index=3, y index=1, header=true]
      {../../Experiments/serial/logs/tms_NEWS20.progress};
      \addlegendentry{L-BFGS}

      \addplot[ultra thick, color=green] table [x index=2, y index=1, header=true]
      {../../Experiments/serial/logs/pmlr_NEWS20.progress};
      \addlegendentry{LC}

    \end{axis}
  \end{tikzpicture}
  \begin{tikzpicture}[scale=0.5]
    \begin{axis}[xmode=log, minor tick num=1,
      title={CLEF dataset: P=1$\times$1$\times$1, $\lambda=0.001$, $\eta=1e2$},
      xlabel={time (secs)}, ylabel={objective}]

      \addplot[ultra thick, color=blue] table [x index=2, y index=1, header=true]
      {../../Experiments/serial/logs/dsmlrnomad_CLEF.progress};      
      \addlegendentry{DS-MLR}
      
      \addplot[ultra thick, color=red] table [x index=3, y index=1, header=true]
      {../../Experiments/serial/logs/tms_CLEF.progress};
      \addlegendentry{L-BFGS}

      \addplot[ultra thick, color=green] table [x index=2, y index=1, header=true]
      {../../Experiments/serial/logs/pmlr_CLEF.progress};
      \addlegendentry{LC}
      
    \end{axis}
  \end{tikzpicture}
     \begin{tikzpicture}[scale=0.5]
    \begin{axis}[xmode=log, minor tick num=1,
      title={LSHTC1-small dataset: P=1$\times$1$\times$1, $\lambda=2.2406e-07$, $\eta=1e4$},
      xlabel={time (secs)}, ylabel={objective}, ymin={0.0}, ymax={1.0}]

      \addplot[ultra thick, color=blue] table [x index=2, y index=1, header=true]
      {../../Experiments/serial/logs/dsmlrnomad_LSHTC1-small.progress};      
      \addlegendentry{DS-MLR}
      
      \addplot[ultra thick, color=red] table [x index=3, y index=1, header=true]
      {../../Experiments/serial/logs/tms_LSHTC1-small-tol1e-12.progress};
      \addlegendentry{L-BFGS}      
      
      \addplot[ultra thick, color=green] table [x index=2, y index=1, header=true]
      {../../Experiments/serial/logs/pmlr_LSHTC1-small.progress};
      \addlegendentry{LC}

    \end{axis}
  \end{tikzpicture}
  \caption{{\bf Data and Model both fit in memory}. In each plot, P=N$\times$M$\times$T denotes that there are $N$ nodes each running $M$ mpi tasks, with $T$ threads each.}
  \label{fig:plottime_smalldatasets}
\end{figure}
L-BFGS is a highly efficient second-order method that has a rapid convergence rate. %In fact, we use the TAO solver which is a super optimized implementation of L-BFGS. Being a second-order method, it has a fast rate of convergence. 
Even when pitched against such a powerful second order method, DS-MLR performs considerably well in comparison. In fact, on some datasets such as NEWS20, DS-MLR is almost on par with L-BFGS in terms of decreasing the objective and also achieves a better f-score much more quickly. Figure \ref{fig:plottime_smalldatasets} shows the progress of objective function as a function of time for DS-MLR, L-BFGS and LC on NEWS20, CLEF, LSHTC1-small datasets. The corresponding plots showing f-score vs time are available in Appendix \ref{sec:additional_plots}. However, L-BFGS loses its applicability when the number of parameters increases beyond what can fit on a single-machine.

DS-MLR consistently shows a faster decrease in objective value compared to LC on all three datasets: NEWS20, LSHTC1-small and CLEF. In fact, LC has a tendency to stall towards the end and progresses very slowly to the optimal objective value. In CLEF dataset, to reach an optimal value of 0.398, DS-MLR takes 1,262 secs while LC takes 21,003 secs. Similarly, in LSHTC1-small, to reach an optimal value of 0.065, DS-MLR takes 1,191 secs while LC takes 32,624 secs.

\subsection{Data Fits and Model Does not Fit}
For this experiment, we compare DS-MLR and LC on LSHTC1-large and ODP datasets.

{\bf LSHTC1-large:} L-BFGS requires all its parameters to fit on one machine and is therefore not suited for model parallelism (even on modestly large datasets such as LSHTC1-large, $\approx$ 4.2 billion parameters need to be stored demanding $\approx$ 34GB). Thus, parallelizing L-BFGS would involve duplicating 34 GB of parameters across all its processors. We ran both DS-MLR and LC using 48 workers. 
%\begin{wrapfigure}[9]{r}{0.6\textwidth}
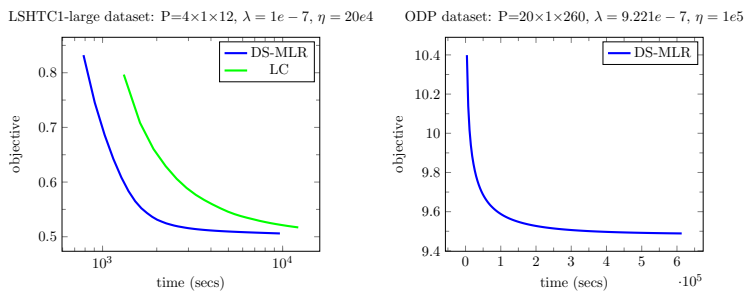
\begin{figure}[H]
\vspace{-1em}
  \centering
   \begin{tikzpicture}[scale=0.5]
    \begin{axis}[xmode=log, minor tick num=1,
      restrict y to domain=0:0.9,
      title={LSHTC1-large dataset: P=4$\times$1$\times$12, $\lambda=1e-7$, $\eta=20e4$},
      xlabel={time (secs)}, ylabel={objective}]

      \addplot[ultra thick, color=blue] table [x index=2, y index=1, header=true]
      {../../Experiments/parallel/lshtclarge_nomad/lshtcl_20e4_rep5_4node_12th.log};
      \addlegendentry{DS-MLR}      
      
      \addplot[ultra thick, color=green] table [x index=2, y index=1, header=true]
      {../../Experiments/parallel/lshtclarge_nomad/pmlr_r1e-7_4node_48proc.log};
      \addlegendentry{LC}
      
    \end{axis}
  \end{tikzpicture}
%  \begin{tikzpicture}[scale=0.41]
%    \begin{axis}[minor tick num=1,
%      %restrict y to domain=0:1.5,
%      title={Wikilarge dataset: P=5$\times$5$\times$20, $\lambda=1e-7$, $\eta=15e5$ },
%      xlabel={time (secs)}, ylabel={objective}]
%
%      \addplot[ultra thick, color=blue] table [x index=2, y index=1, header=true]
%      {../../Experiments/parallel/wikilarge_nomad/wikilarge_res_5_np25_20_100_1e-7_15e5.log};
%      \addlegendentry{DS-MLR}
%
%    \end{axis}
%  \end{tikzpicture}
  \begin{tikzpicture}[scale=0.5]
    \begin{axis}[minor tick num=1,
      %restrict y to domain=10.36:10.37,
      %restrict x to domain=1e4:50e4,
      title={ODP dataset: P=20$\times$1$\times$260, $\lambda=9.221e-7$, $\eta=1e5$},
      xlabel={time (secs)}, ylabel={objective}]

      \addplot[ultra thick, color=blue] table [x index=2, y index=1, header=true]
      {../../Experiments/parallel/odp/odp_nomaddsmlr_260_np20_1_1e5_9.221e-7_packet_10.log};
      \addlegendentry{DS-MLR}

    \end{axis}
  \end{tikzpicture}  
  \caption{{\bf Data Fits and Model does not fit}. In each plot, P=N$\times$M$\times$T denotes that there are $N$ nodes each running $M$ mpi tasks, with $T$ threads each.}
  \label{fig:plottime_lshtclarge_wikilarge}
\end{figure}
%\end{wrapfigure}
Figure \ref{fig:plottime_lshtclarge_wikilarge} (left) shows how the objective function changes vs time for DS-MLR and LC. As can be seen, DS-MLR out performs LC by a wide-margin despite the advantage LC has by duplicating data across all its processors.\\
{\bf ODP}:
We ran DS-MLR on the ODP dataset \footnote{\url{https://github.com/JohnLangford/vowpal_wabbit/tree/master/demo/recall_tree/odp}} which has a huge model parameter size of 355 GB. For this experiment we used 20 nodes $\times$ 1 mpi task $\times$ 260 threads. The progress in decreasing the objective function value is shown in Figure \ref{fig:plottime_lshtclarge_wikilarge} (right). LC method being a second-order method has a very high per-iteration cost and it takes an enormous amount of time to finish even a single iteration.

\subsection{Data Does not Fit and Model Fits}
{\bf YouTube8M-Video}: This dataset was created by pre-processing the publicly available dataset of youtube video embeddings \footnote{\url{https://research.google.com/youtube8m/}} into a multi-class classification dataset consisting of 4,716 classes and 1,152 features. Since it was created from features derived from embeddings, it is a perfectly dense dataset. 
%\begin{wrapfigure}[7]{r}{0.45\textwidth}
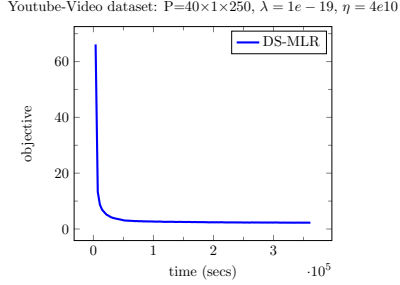
\begin{figure}[H]
  \centering
    \begin{tikzpicture}[scale=0.5]
    \begin{axis}[minor tick num=1,
      %restrict y to domain=10.36:10.37,
      %restrict x to domain=1e4:50e4,
      title={Youtube-Video dataset: P=40$\times$1$\times$250, $\lambda=1e-19$, $\eta=4e10$},
      xlabel={time (secs)}, ylabel={objective}]

      \addplot[ultra thick, color=blue] table [x index=2, y index=1, header=true]
      {../../Experiments/parallel/youtube_video/yt8m_nomaddsmlr_260_np4_1_1e5_2.039e-7_packet_10.log};
      \addlegendentry{DS-MLR}

%      \addplot[ultra thick, color=red] table [x index=2, y index=1, header=true]
%      {../../Experiments/parallel/youtube_video/yt8m_nomaddsmlr_260_np4_1_5e5_2.039e-7_packet_10.log};
%      \addlegendentry{DS-MLR}      

    \end{axis}
  \end{tikzpicture}
  \caption{{\bf Data does not fit and Model fits}. In each plot, P=N$\times$M$\times$T denotes that there are $N$ nodes each running $M$ mpi tasks, with $T$ threads each.}
  \label{fig:plottime_youtube}
\end{figure}
%\end{wrapfigure}
We used the configuration of 20 nodes $\times$ 1 mpi tasks $\times$ 260 threads to run DS-MLR on this dataset and we observed a fast convergence as shown in Figure \ref{fig:plottime_youtube}. This is likely because DS-MLR being non-blocking and asynchronous in nature runs at its peak performance on a dense dataset like YouTube8M-Video, since the number of non-zeros in the data remains uniform across all its workers.

\subsection{Data Does not Fit and Model Does not Fit}
{\bf Reddit datasets}:
In this sub-section, we demonstrate the capability of DS-MLR to solve a multi-class classification problem of massive scale, using a bag-of-words dataset {\it RedditFull} created out of 1.7 billion reddit user comments spanning the period 2007-2015. 
%\begin{wrapfigure}[11]{r}{0.6\textwidth}
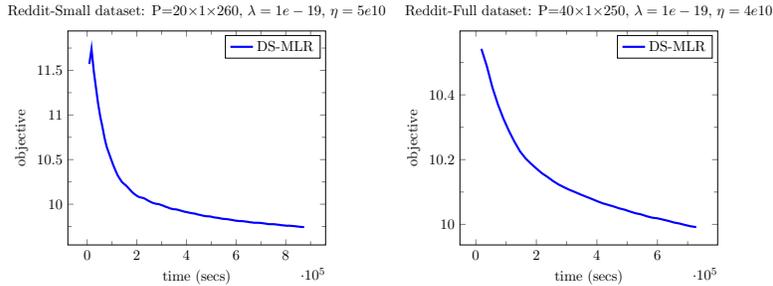
\begin{figure}[H]
  \centering
    \begin{tikzpicture}[scale=0.5]
    \begin{axis}[minor tick num=1,
      %restrict y to domain=10.36:10.37,
      %restrict x to domain=1e4:50e4,
      title={Reddit-Small dataset: P=20$\times$1$\times$260, $\lambda=1e-19$, $\eta=5e10$},
      xlabel={time (secs)}, ylabel={objective}]

      \addplot[ultra thick, color=blue] table [x index=2, y index=1, header=true]
      {../../Experiments/parallel/redditsmall/small_reddit_data_newscript_260_np20_1_5e10_1e-19_packet_10.log};
      \addlegendentry{DS-MLR}

    \end{axis}
  \end{tikzpicture}
    \begin{tikzpicture}[scale=0.5]
    \begin{axis}[minor tick num=1,
      %restrict y to domain=10.36:10.37,
      %restrict x to domain=1e4:50e4,
      title={Reddit-Full dataset: P=40$\times$1$\times$250, $\lambda=1e-19$, $\eta=4e10$},
      xlabel={time (secs)}, ylabel={objective}]

      \addplot[ultra thick, color=blue] table [x index=2, y index=1, header=true]
      {../../Experiments/parallel/redditfull/redditfull_40_250_np40_1_4e10_1e-19.log};
      \addlegendentry{DS-MLR}

    \end{axis}
  \end{tikzpicture}  
  
  \caption{{\bf Data does not fit and Model does not fit}. In each plot, P=N$\times$M$\times$T denotes that there are $N$ nodes each running $M$ mpi tasks, with $T$ threads each.}
  \label{fig:plottime_reddit}
\end{figure}
%\end{wrapfigure}
Our aim is to classify a particular reddit comment (data point) into a suitable sub-reddit (class). The data and model parameters occupy 200 GB and 300 GB respectively. Therefore, both L-BFGS and LC cannot be applied here. 
We also created a smaller subset of this dataset {\it Reddit-Small} by sub-sampling around 50 million data points. The result of running DS-MLR on these are shown in Figure \ref{fig:plottime_reddit}. This experiment corresponds to the last scenario in Table \ref{table:data_model_parallel_compact} where simultaneous data and model parallelism is inevitable.

In Appendix \ref{sec:rankdistrib_plots} we study the predictive quality of DS-MLR and Appendix \ref{sec:additional_plots} has additional plots showing progress of f-score vs time.

%Appendix \ref{sec:scaling_dsmlr} presents DS-MLR's scaling behavior as the number of workers are varied and Appendix \ref{sec:additional_plots} has additional plots showing progress of f-score vs time.
%As a side note, it is also worth noting that DS-MLR does not use any step-size tuning or variance reduction methods like {\em AdaGrad} \citep{duchi2011adaptive} or {\em SVRG} \citep{JohZha}, which also partly contributes to its memory efficiency.

%%%%%%%%%%%%%%%%%%%%%%%%%%%%%%%%%%%%%%%%%
\section{Scaling behavior of DS-MLR}
\label{sec:scaling}
In Figure \ref{fig:scaling}, we look at the speedup curves obtained on LSHTC1-large as the number of workers are varied as $1, 2, 4, 8, 16, 20$. The ideal speedup which corresponds to a linear speedup is denoted in red. %The scaling behavior of DS-MLR in pretty reasonable as the number of workers are varied.

\begin{figure}[h]
  \centering
  \begin{tikzpicture}[scale=0.50]
    \begin{axis}[minor tick num=1,
      legend style={at={(-0.2,0.85)}, anchor= south west},
      title={LSHTC1-large},
      xlabel={number of workers}, ylabel={speedup}]

      \addplot[thin, color=red, mark=*] table [x index=0, y index=0, header=true]
      {../../Experiments/parallel/lshtclarge_singlenode/scaling_threads/speedup.log};

      \addplot[thin, color=blue, mark=*] table [x index=0, y expr=\thisrow{time_s}/\thisrow{time_p}, header=true]
      {../../Experiments/parallel/lshtclarge_singlenode/scaling_threads/speedup.log};

      \legend{ideal scaling (linear speedup), DS-MLR}

    \end{axis}
  \end{tikzpicture}
%  \begin{tikzpicture}[scale=0.50]
%    \begin{axis}[minor tick num=1,
%      legend style={at={(-0.2,0.85)}, anchor= south west},
%      title={YouTube8M-Video},
%      xlabel={number of workers}, ylabel={speedup}]
%
%      \addplot[thin, color=red, mark=*] table [x index=0, y expr=\thisrow{num_p}/30, header=true]
%      {../../Experiments/parallel/youtube_video/scaling/speedup.log};
%
%      \addplot[thin, color=blue, mark=*] table [x index=0, y expr=\thisrow{time_s}/\thisrow{time_p}, header=true]
%      {../../Experiments/parallel/youtube_video/scaling/speedup.log};
%
%      \legend{ideal scaling (linear speedup), DS-MLR}
%
%    \end{axis}
%  \end{tikzpicture}  
%  \caption{Test micro f1-score as a function of computation time (time in seconds x number of cores), when the number of cores and threads are varied}  
 % \caption{Change in objective function and f1-scores as a function of computation time while the \# of workers (threads) are varied}  
 \caption{Scalability analysis of DS-MLR on LSHTC1-large}  
  %% BUGBUG: Mention the # of nodes and # of cores in the title of each plot
  %% Change plot lines to that like in nomad-lda paper
  %% Move dataset references to Reference section and cite them (instead of using footnotes)
  \label{fig:scaling}
\end{figure}
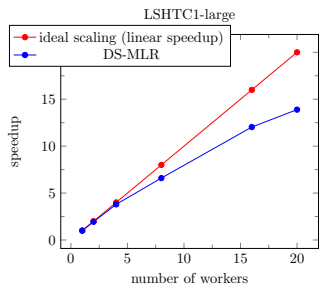

%%%%%%%%%%%%%%%%%%%%%%%%%%%%%%%%%%%%%%%%%

\section{Conclusion}
\label{sec:Conclusion}
In this paper, we present a new stochastic optimization algorithm (DS-MLR) to solve multinomial logistic regression problems having large number of examples and classes, by a reformulation that makes it both data and model parallel simultaneously. As a result, DS-MLR can scale to arbitrarily large datasets where to the best of our knowledge, many of the existing distributed algorithms cannot be applied. Our algorithm is distributed, asynchronous, non-blocking and avoids any bulk-synchronization overheads. We provide empirical results showing DS-MLR applies to all regimes of distributed machine learning, especially the case where both data and model sizes exceed the memory capacity of a single machine. We show this on an extreme multi-class classification reddit dataset consisting of 200 GB data and 300 GB parameters respectively. In terms of future work, DS-MLR has several possible extensions such as extreme multi-label classification.

\clearpage
%\section{Acknowledgments}
\bibliographystyle{abbrv}
\bibliography{mlr}

%%%%%%%%%%%%%%%%%%%%%%%%%%%%%%%%%%%%%%%%%

\newpage
\appendix

In the following sections, we provide a more detailed proof of convergence for our algorithm and also include additional plots from our empirical study.

\section{Rates of convergence}
\label{sec:app_rate}
%\subsection{DS-MLR 1-delay}
First the diameter of $W$ space can be bounded by a universal constant (independent of $N, D, K$) because we can always enforce that
$\frac{1}{2\lambda} \nbr{W}^2 \le f(\zero) = \log K$ (ignoring log term).
We also assume all $\xb_{i}$ are bounded in $L_2$ norm by some constant $r$.
We will write $r$ as a constant everywhere.  They are not necessarily equal; in fact we may write $r^2$ and $2r$ as $r$.  It just stands for some constant that is independent of $\epsilon, D, N$ and $K$.

We index outer iteration by superscript $t$ and inner-epochs within each outer iteration by subscript $k$.
So $W^t_1 = W^{t-1}_{N+1}$, which we also denote as $W^t$.
We consider optimizing the objective
\begin{align}
  L_1(W)=F(W) &= \frac{1}{N} \sum_{i=1}^N f_i(W),
\end{align}
where $f_i(W) = \frac{\lambda}{2} \nbr{W}^2 - w^T_{y_i} \xb_{i} + \log \sum_{k=1}^K \exp(\wb_{k}^T \xb_{i})$.
Clearly $f_i$ has a variational representation
\begin{align}
\label{eq:def_fi_var}
  f_i(W) & = \frac{\lambda}{2} \nbr{W}^2 - w^T_{y_i} \xb_{i} \nonumber \\
&  + \min_{a_i \in \RR} \cbr{- a_i + \sum_{k=1}^K \exp(\wb_{k}^T \xb_{i} + a_i)} -1,
\end{align}
where the optimal $a_i$ is attained at $-\log \sum_{k=1}^K \exp(\wb_{k}^T \xb_{i})$.
So given $W$, we can first compute the optimal $a_i$,
and then use it to compute the gradient of $f_i$ via the variational form (Danskin's theorem \citep{bertsekas1999nonlinear}).

\begin{align}
\label{eq:Danskin}
  %\frac{\partial}{\partial w_c} f_i(W) = \lambda w_c - \llbracket y_i = c \rrbracket \xb_{i} + \exp(w_c^T \xb_{i} + a_i) \xb_{i}.
  \frac{\partial}{\partial \wb_{k}} f_i(W) &= \lambda \wb_{k} - [ y_i  = k ] \xb_{i} + \exp(\wb_{k}^T \xb_{i} + a_i) \xb_{i}.
\end{align}
Here $[ \cdot ] = 1$ if $\cdot$ is true, and 0 otherwise.

Due to the distributed setting, we are only able to update $a_i$ to their optimal value at the end of each epoch (\ie\ based on $W^t$):
\begin{align}
\label{eq:def_ait}
  a_i^t = a_i(W^t) = - \log \sum_{k=1}^K \exp(\xb_{i}^T w^t_k).
\end{align}
We are \emph{not} able to compute the optimal $a_i$ for the latest $W$ when incremental gradient is performed through the whole dataset.
Fortunately, since $W$ is updated in an epoch by a fixed (small) step size $\eta_t$,
it is conceivable that the $a_i$ computed from $W^t$ will not be too bad as a solution in \eqref{eq:def_fi_var} for $W^t_k$, $k \in [m]$.
In fact, if $\nbr{W^t_k - W^t}$ is order $O(\eta_t)$,
then the following Lemma says the gradient computed from \eqref{eq:Danskin} using the out-of-date $a_i$ is also $O(\eta_t)$ away from the true gradient at $W^t_k$.

\begin{lemma}
\label{lem:gradient_error}
%  Suppose $\nbr{W^t_{k} - W^t} \le \eta_t r$.
  Denote the approximate gradient of $f_i$ evaluated at $W^t_k$ based on $a_i^t$ as
  \begin{align}
    \Gtil^t_k = (\gbtil_1, \ldots, \gbtil_K),
  \end{align}
  where $\gbtil_c = \lambda w^t_{k,c} - [ y_i = c ] \xb_{i} + \exp(\xb_{i}^T w^t_{k,c}  + a^t_i) \xb_{i}$.
  Then $\nbr{\Gtil^t_k - \grad_W f_i(W^t_k)} \le \frac{r}{K} \nbr{W^t_{k} - W^t}$.
%  Note we use $r$ as a constant, which may stand for different values at different places.
\end{lemma}

\begin{proof}
  Unfolding the term $a_i^t$ from \eqref{eq:def_ait},
  \begin{align*}
    \gbtil_c - \frac{\partial}{\partial w_c} f_i(W^t_k) 
    &= \left( \frac{\exp(\xb_{i}^T w^t_{k,c})}{\sum_{c=1}^K \exp(\xb_{i}^T w^t_c)} - \frac{\exp(\xb_{i}^T w^t_{k,c})}{\sum_{c=1}^K \exp(\xb_{i}^T w^t_{k,c})} \right) \xb_{i}
  \end{align*}
  Therefore
  \begin{align*}
   & \nbr{\Gtil - \grad_W f_i(W^t_k)} \nonumber \\
   & \le r \sqrt{K} \abr{\frac{1}{\sum_{c=1}^K \exp(\xb_{i}^T w^t_c)} - \frac{1}{\sum_{c=1}^K \exp(\xb_{i}^T w^t_{k,c})}}
  \end{align*}
  So it suffices to upper bound the gradient of $1/\sum_{c=1}^K \exp(\xb_{i}^T w_c)$.
  Since $\xb_{i}$ and $w_c$ are bounded, $\exp(\xb_{i}^T w_c)$ is lower bounded by a positive universal constant\footnote{If one is really really meticulous and notes that $\nbr{W}^2 \le 2\lambda \log K$ which does involve $K$, one should be appeased that $\exp (\sqrt{\log K})$ is $o(K^\alpha)$ for any $\alpha > 0$.}. Now,
  \begin{align*}
    & \nbr{\grad_W \frac{1}{\sum_{c=1}^K \exp(\xb_{i}^T w_c)}} \nonumber \\
    &= \frac{1}{(\sum_{c=1}^K \exp(\xb_{i}^T w_c))^2}\nbr{  (\exp(\xb_{i}^T w_1) \xb_{i}, \ldots, \exp(\xb_{i}^T w_K) x_i)} \\
    &\le \frac{\sqrt{K}}{K^2} r
  \end{align*}
\end{proof}
Using Lemma \ref{lem:gradient_error}, we can now show that our algorithm achieves $O(1/\epsilon^2)$ epoch complexity, with no dependency on $m$, $d$, or $K$.
In fact we just apply Nedic's algorithm \citep{Ned01} and analysis on $F(W)$.
However we need to adapt their proof a little because they assume the gradients are \emph{exact}.

First we need to bound some quantities.
$\nbr{\grad f_i(W)} \le r$ because $W$ is bounded, and for $K$ numbers $p_1, \ldots, p_K$ on a simplex with $\sum_c p_c = 1$, we have $\sum_c p^2_c \le 1$.
Without loss of generality, suppose $f_k$ is used for update at step $k$.
Then $W^t_k$ is subtracted by $\frac{\eta_t}{m} (\lambda W^t_k - x_k \otimes \evec'_{y_k} + \Gtil^t_k)$, where $\otimes$ is Kroneker product and $\evec_c$ is a canonical vector.
As long as $\eta_t \le \frac{1}{\lambda}$, we can recursively apply Lemma \ref{lem:gradient_error}
and derive bounds
\begin{align}
\label{eq:diff_Wk}
  \nbr{W^t_k - W^t} &\le \frac{k}{m} \eta_t r, \\
  \nbr{\grad_W f_k(W_k^t)-\Gtil^t_k} &\le \eta_t r, \\
  \nbr{\Gtil^t_k} &\le r,
\end{align}
for all $k$. %\footnote{This is obvious when $\lambda = 0$.  For $\lambda > 0$, see my proof for RERM.}
Now we run Nedic's proof.
Then for any $W$
\begin{align*}
  & \nbr{W^t_{k+1} - W}^2 \nonumber \\
  &= \nbr{W^t_{k} - \frac{\eta_t}{m} \Gtil^t_k - W}^2 \\
  &= \nbr{W^t_{k} - W}^2 - 2 \frac{\eta_t}{m} \inner{\Gtil^t_k}{W^t_{k}-W} + \frac{\eta_t^2}{m^2} \nbr{\Gtil^t_k}^2 \\
  &= \nbr{W^t_{k} - W}^2 - 2 \frac{\eta_t}{m} \Bigg(\inner{\grad_W f_k(W_k^t)}{W^t_{k}-W} + \\
  &\qquad \qquad \inner{\Gtil^t_k-\grad_W f_k(W_k^t)}{W^t_{k}-W} \Bigg) + \frac{\eta_t^2}{m^2} \nbr{\Gtil^t_k}^2 \\
  &\le \nbr{W^t_{k} - W}^2 - 2 \frac{\eta_t}{m} \rbr{f_k(W^t_k) - f_k(W) - \eta_t r} + \frac{\eta_t^2}{m^2} r^2.
\end{align*}

Telescoping over $k=1, \ldots, m$, we obtain that for all $W$ and $t$:
\begin{align*}
  & \nbr{W^{t+1}-W}^2 \nonumber \\
  & \le \nbr{W^t-W}^2
  - 2 \frac{\eta_t}{m} \sum_{k=1}^{m} \rbr{f_k(W^t_k) - f_k(W)} + \eta_t^2 r \\
  & \le \nbr{W^t-W}^2
  \nonumber \\
  & - 2 \eta_t \rbr{F(W^t) - F(W) + \frac{1}{m} \sum_{k=1}^{m} \rbr{f_k(W^t_{k}) - f_k(W^t)}} 
  + \eta_t^2 r.
\end{align*}
Using the fact that $\grad f_k$ is bounded by a universal constant, we further derive
\begin{align*}
  \nbr{W^{t+1}-W}^2 
  &\le \nbr{W^t-W}^2 - 2 \eta_t \rbr{F(W^t) - F(W)} 
  \nonumber \\
  &\phantom{=}+ 2 \frac{\eta_t}{m} r \sum_{k=1}^{m} \nbr{W^t_k - W^t} + \eta_t^2 r \\
  &\le \nbr{W^t-W}^2 - 2 \eta_t \rbr{F(W^t) - F(W)}
  \nonumber \\
  &\phantom{=}+ 2 \frac{\eta_t^2}{m} r \sum_{k=1}^{m} \frac{k}{m} + \eta_t^2 r \quad (\text{by } \eqref{eq:diff_Wk})\\
  &= \nbr{W^t-W}^2 - 2 \eta_t \rbr{F(W^t) - F(W)} + \eta_t^2 r.
\end{align*}
Now use the standard step size of $O(1/\sqrt{t})$, we conclude
\begin{align}
  \min_{t=1...T} F(W^t) - F(W) \le \frac{r}{\sqrt{T}}.
\end{align}
Note the proof has not used the convexity of $a_i$ in \eqref{eq:def_fi_var} at all.
This is reasonable because it is ``optimized out".

%%%%%%%%%%%%%%%%%%%%%%%%%%%%%%%%%%%%%%%%%
\section{Rank Distribution}
\label{sec:rankdistrib_plots}
In this section, we plot the cumulation distribution of ranks of test labels. This is a proxy 
for the precision@k curve and gives a more closer indication of the predictive performance of a multinomial classification 
algorithm. In Figures \ref{fig:plotcdf_small} and \ref{fig:plotcdf_large}, we plot the precision obtained after the first 5 iterations (denoted by dashed lines), 
and after the end of optimization (denoted by solid lines). As seen, DS-MLR performs competitively and in general tends to give a good accuracy within the first 5 iterations. %As can be seen, just using roughly 1/4-th of the top-k classes was enough to get an accuracy of at least ~0.95 in all datasets.

\begin{figure}[h]
 \centering
  \begin{tikzpicture}[scale=0.50]
    \begin{axis}[minor tick num=1,
      legend style={at={(0.55,0.3)}, anchor= west},
      title={CLEF},
      xlabel={K (number of classes)}, ylabel={CDF of rank distribution}, ymin={0.9}, ymax={1}]

      \addplot[thick, dashed, color=blue] table [x index=0, y index=1, header=false]
      {../../Experiments/notes/logs_with_ranks/cdf_CLEF_it5.txt};
      \addlegendentry{DS-MLR iter 5}
      
      \addplot[thick, color=blue] table [x index=0, y index=1, header=false]
      {../../Experiments/notes/logs_with_ranks/cdf_CLEF_it10000.txt};
      \addlegendentry{DS-MLR end}
      
      \addplot[thick, dashed, color=red] table [x index=0, y index=2, header=false]
      {../../Experiments/notes/logs_with_ranks/cdf_CLEF_it5.txt};
      \addlegendentry{L-BFGS iter 5}
      
      \addplot[thick, color=red] table [x index=0, y index=2, header=false]
      {../../Experiments/notes/logs_with_ranks/cdf_CLEF_it10000.txt};
      \addlegendentry{L-BFGS end}

      \addplot[thick, dashed, color=green] table [x index=0, y index=3, header=false]
      {../../Experiments/notes/logs_with_ranks/cdf_CLEF_it5.txt};
      \addlegendentry{LC iter 5}
      
      \addplot[thick, color=green] table [x index=0, y index=3, header=false]
      {../../Experiments/notes/logs_with_ranks/cdf_CLEF_it10000.txt};
      \addlegendentry{LC end}
    \end{axis}
  \end{tikzpicture}
\begin{tikzpicture}[scale=0.50]
    \begin{axis}[minor tick num=1,
      legend style={at={(0.55,0.3)}, anchor= west},
      title={NEWS20},
      xlabel={K (number of classes)}, ylabel={CDF of rank distribution}, ymin={0.85}, ymax={1}]

      \addplot[thick, dashed, color=blue] table [x index=0, y index=1, header=false]
      {../../Experiments/notes/logs_with_ranks/cdf_NEWS20_it5.txt};
      \addlegendentry{DS-MLR iter 5}

      \addplot[thick, color=blue] table [x index=0, y index=1, header=false]
      {../../Experiments/notes/logs_with_ranks/cdf_NEWS20_it10000.txt};
      \addlegendentry{DS-MLR end}

      \addplot[thick, dashed, color=red] table [x index=0, y index=2, header=false]
      {../../Experiments/notes/logs_with_ranks/cdf_NEWS20_it5.txt};
      \addlegendentry{L-BFGS iter 5}
      
      \addplot[thick, color=red] table [x index=0, y index=2, header=false]
      {../../Experiments/notes/logs_with_ranks/cdf_NEWS20_it10000.txt};
      \addlegendentry{L-BFGS end}

      \addplot[thick, dashed, color=green] table [x index=0, y index=3, header=false]
      {../../Experiments/notes/logs_with_ranks/cdf_NEWS20_it5.txt};
      \addlegendentry{LC iter 5}

      \addplot[thick, color=green] table [x index=0, y index=3, header=false]
      {../../Experiments/notes/logs_with_ranks/cdf_NEWS20_it10000.txt};
      \addlegendentry{LC end}
    \end{axis}
  \end{tikzpicture}
\begin{tikzpicture}[scale=0.50]
    \begin{axis}[minor tick num=1,
      legend style={at={(0.55,0.3)}, anchor= west},
      title={LSHTC1-small},
      xlabel={K (number of classes)}, ylabel={CDF of rank distribution}, ymin={0.25}, ymax={1}]

      \addplot[thick, dashed, color=blue] table [x index=0, y index=1, header=false]
      {../../Experiments/notes/logs_with_ranks/cdf_LSHTC1-small_it5.txt};
      \addlegendentry{DS-MLR iter 5}

      \addplot[thick, color=blue] table [x index=0, y index=1, header=false]
      {../../Experiments/notes/logs_with_ranks/cdf_LSHTC1-small_it10000.txt};
      \addlegendentry{DS-MLR end}

      \addplot[thick, dashed, color=red] table [x index=0, y index=2, header=false]
      {../../Experiments/notes/logs_with_ranks/cdf_LSHTC1-small_it5.txt};
      \addlegendentry{L-BFGS iter 5}
      
      \addplot[thick, color=red] table [x index=0, y index=2, header=false]
      {../../Experiments/notes/logs_with_ranks/cdf_LSHTC1-small_it10000.txt};
      \addlegendentry{L-BFGS end}

      \addplot[thick, dashed, color=green] table [x index=0, y index=3, header=false]
      {../../Experiments/notes/logs_with_ranks/cdf_LSHTC1-small_it5.txt};
      \addlegendentry{LC iter 5}

      \addplot[thick, color=green] table [x index=0, y index=3, header=false]
      {../../Experiments/notes/logs_with_ranks/cdf_LSHTC1-small_it10000.txt};
      \addlegendentry{LC end}
    \end{axis}
  \end{tikzpicture}
  \caption{Cumulative distribution of predictive ranks of the test labels for the three small datasets}
  \label{fig:plotcdf_small}
\end{figure}
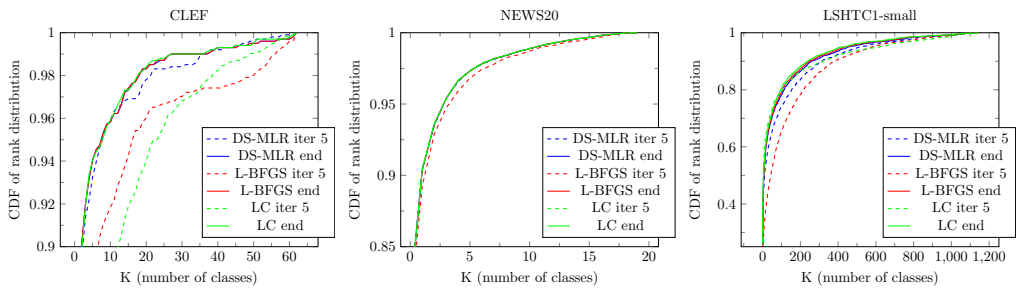

\begin{figure}[h]
 \centering
  \begin{tikzpicture}[scale=0.50]
    \begin{axis}[minor tick num=1,
      legend style={at={(0.55,0.3)}, anchor= west},
      title={LSHTC1-large},
      xlabel={K (number of classes)}, ylabel={CDF of rank distribution}, ymin={0.3}, ymax={1}]

      \addplot[thick, dashed, color=blue] table [x index=0, y index=1, header=false]
      {../../Experiments/notes/logs_with_ranks/cdf_LSHTC1-large_it5.txt};
      \addlegendentry{DS-MLR iter 5}

      \addplot[thick, color=blue] table [x index=0, y index=1, header=false]
      {../../Experiments/notes/logs_with_ranks/cdf_LSHTC1-large_it109.txt};
      \addlegendentry{DS-MLR end}
      
%      \addplot[thick, color=red] table [x index=0, y index=2, header=false]
%      {../../Experiments/notes/logs_with_ranks/cdf_LSHTC1-large_it109.txt};
%      \addlegendentry{L-BFGS}
%
%      \addplot[thick, color=green] table [x index=0, y index=3, header=false]
%      {../../Experiments/notes/logs_with_ranks/cdf_LSHTC1-large_it109.txt};
%      \addlegendentry{LC}
    \end{axis}
  \end{tikzpicture}
%  \begin{tikzpicture}[scale=0.50]
%    \begin{axis}[xmode=log, minor tick num=1,
%      legend style={at={(0.55,0.2)}, anchor= west},
%      title={Reddit 20M subset},
%      xlabel={K (number of classes)}, ylabel={CDF of rank distribution}]
%
%%      \addplot[thick, color=blue] table [x index=2, y index=5, header=true]
%%      {../../Experiments/serial/logs/dsmlrnomad_NEWS20.progress};
%%      \addlegendentry{DS-MLR}
%      
%    \end{axis}
%  \end{tikzpicture}
  %\caption{CDF of the distribution of true ranks of the test labels for the larger datasets: LSHTC1-large and Reddit-20M subset}
  \caption{Cumulative distribution of predictive ranks of the test labels for the larger datasets: LSHTC1-large}
  \label{fig:plotcdf_large}
\end{figure}
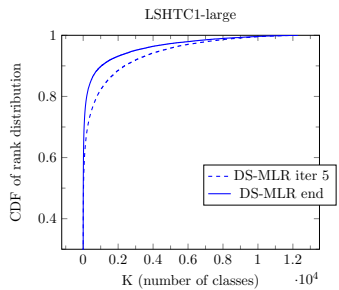

%%%%%%%%%%%%%%%%%%%%%%%%%%%%%%%%%%%%%%%%%
\newpage
\section{Additional Plots}
\label{sec:additional_plots}
In this section, we show how the macro and micro f-score change as a function of time on the various datasets
reported in Table \ref{table:datasets}.

%%%% NEWS20 %%%%%%
\begin{figure*}[h]
 \centering
  \begin{tikzpicture}[scale=0.50]
    \begin{axis}[xmode=log, minor tick num=1,
      legend style={at={(0.55,0.2)}, anchor= west},
      title={NEWS20},
      xlabel={time (secs)}, ylabel={test micro F1}, ymin={0.75}, ymax={0.85}]

      \addplot[thick, color=blue] table [x index=2, y index=6, header=true]
      {../../Experiments/serial/logs/dsmlrnomad_NEWS20.progress};
      \addlegendentry{DS-MLR}
      
      \addplot[thick, color=red] table [x index=3, y index=7, header=true]
      {../../Experiments/serial/logs/tms_NEWS20.progress};
      \addlegendentry{L-BFGS}

      \addplot[thick, color=green] table [x index=2, y index=6, header=true]
      {../../Experiments/serial/logs/pmlr_NEWS20.progress};
      \addlegendentry{LC}
      
    \end{axis}
  \end{tikzpicture}
  \begin{tikzpicture}[scale=0.50]
    \begin{axis}[xmode=log, minor tick num=1,
      legend style={at={(0.55,0.2)}, anchor= west},
      title={NEWS20},
      xlabel={time (secs)}, ylabel={test macro F1}, ymin={0.75}, ymax={0.85}]

      \addplot[thick, color=blue] table [x index=2, y index=5, header=true]
      {../../Experiments/serial/logs/dsmlrnomad_NEWS20.progress};
      \addlegendentry{DS-MLR}
      
      \addplot[thick, color=red] table [x index=3, y index=6, header=true]
      {../../Experiments/serial/logs/tms_NEWS20.progress};
      \addlegendentry{L-BFGS}
      
      \addplot[thick, color=green] table [x index=2, y index=5, header=true]
      {../../Experiments/serial/logs/pmlr_NEWS20.progress};
      \addlegendentry{LC}

    \end{axis}
  \end{tikzpicture}
  \caption{(Left): test micro F1 vs time, (Right): test macro F1 vs time}
  \label{fig:plotfscore_news20}
\end{figure*}

%%%% CLEF %%%%%%

\begin{figure*}[h]
  \centering
  \begin{tikzpicture}[scale=0.50]
    \begin{axis}[xmode=log, minor tick num=1,
      legend style={at={(0.55,0.2)}, anchor= west},
      title={CLEF},
      xlabel={time (secs)}, ylabel={test micro F1}, ymin={0.6}, ymax={0.9}]

      \addplot[thick, color=blue] table [x index=2, y index=6, header=true]
      {../../Experiments/serial/logs/dsmlrnomad_CLEF.progress};
      \addlegendentry{DS-MLR}
      
      \addplot[thick, color=red] table [x index=3, y index=7, header=true]
      {../../Experiments/serial/logs/tms_CLEF.progress};
      \addlegendentry{L-BFGS}
      
      \addplot[thick, color=green] table [x index=2, y index=6, header=true]
      {../../Experiments/serial/logs/pmlr_CLEF.progress};
      \addlegendentry{LC}
            
    \end{axis}
  \end{tikzpicture}
  \begin{tikzpicture}[scale=0.50]
    \begin{axis}[xmode=log, minor tick num=1,
      legend style={at={(0.55,0.2)}, anchor= west},
      title={CLEF},
      xlabel={time (secs)}, ylabel={test macro F1}, ymin={0.3}, ymax={0.7}]

      \addplot[thick, color=blue] table [x index=2, y index=5, header=true]
      {../../Experiments/serial/logs/dsmlrnomad_CLEF.progress};
      \addlegendentry{DS-MLR}
      
      \addplot[thick, color=red] table [x index=3, y index=6, header=true]
      {../../Experiments/serial/logs/tms_CLEF.progress};
      \addlegendentry{L-BFGS}
      
      \addplot[thick, color=green] table [x index=2, y index=5, header=true]
      {../../Experiments/serial/logs/pmlr_CLEF.progress};
      \addlegendentry{LC}
            
    \end{axis}
  \end{tikzpicture}
  \caption{(Left): test micro F1 vs time, (Right): test macro F1 vs time}
  \label{fig:plotfscore_clef}
 \end{figure*}

%%%% lshtc1-small %%%%%%

\begin{figure*}[h]
  \centering
  \begin{tikzpicture}[scale=0.50]
    \begin{axis}[xmode=log, minor tick num=1,
      title={LSHTC1-small},
      xlabel={time (secs)}, ylabel={test micro F1}, ymin={0.0}, ymax={0.5}]

      \addplot[ultra thick, color=blue] table [x index=2, y index=6, header=true]
      {../../Experiments/serial/logs/dsmlrnomad_LSHTC1-small.progress};
      \addlegendentry{DS-MLR}
      
%      \addplot[ultra thick, color=red] table [x index=3, y index=7, header=true]
%      {../../Experiments/serial/logs/tms_LSHTC1-small.progress};
%      \addlegendentry{L-BFGS}
      
      \addplot[ultra thick, color=red] table [x index=3, y index=7, header=true]
      {../../Experiments/serial/logs/tms_LSHTC1-small-tol1e-12.progress};
      \addlegendentry{L-BFGS}
      
      \addplot[ultra thick, color=green] table [x index=2, y index=6, header=true]
      {../../Experiments/serial/logs/pmlr_LSHTC1-small.progress};
      \addlegendentry{LC}
      
    \end{axis}
  \end{tikzpicture}
  \begin{tikzpicture}[scale=0.50]
    \begin{axis}[xmode=log,
      title={LSHTC1-small},
      xlabel={time (secs)}, ylabel={test macro F1}, ymin={0.0}, ymax={0.3}]
      
      \addplot[ultra thick, color=blue] table [x index=2, y index=5, header=true]
      {../../Experiments/serial/logs/dsmlrnomad_LSHTC1-small.progress};
      \addlegendentry{DS-MLR}
            
%      \addplot[ultra thick, color=red] table [x index=3, y index=6, header=true]
%      {../../Experiments/serial/logs/tms_LSHTC1-small.progress};
%      \addlegendentry{L-BFGS}

      \addplot[ultra thick, color=red] table [x index=3, y index=6, header=true]
      {../../Experiments/serial/logs/tms_LSHTC1-small-tol1e-12.progress};
      \addlegendentry{L-BFGS}
      
      \addplot[ultra thick, color=green] table [x index=2, y index=5, header=true]
      {../../Experiments/serial/logs/pmlr_LSHTC1-small.progress};
      \addlegendentry{LC}
            
    \end{axis}
  \end{tikzpicture}
  \caption{(Left): test micro F1 vs time, (Right): test macro F1 vs time}
  \label{fig:plotfscore_lshtc1small}
\end{figure*}

%%%% lshtc1-large %%%%%%

\begin{figure*}
  \centering
   \begin{tikzpicture}[scale=0.50]
    \begin{axis}[xmode=log, minor tick num=1,
      title={LSHTC1-large},
%      restrict y to domain=0.28:0.5,
      xlabel={time (secs)}, ylabel={test micro F1}, ymin={0.26}, ymax={0.33}]

      \addplot[ultra thick, color=blue] table [x index=2, y index=6, header=true]
      {../../Experiments/parallel/lshtclarge_nomad/lshtcl_12e4_rep5_4node_12th.log};
      \addlegendentry{DS-MLR}
      
      \addplot[ultra thick, color=green] table [x index=2, y index=6, header=true]
      {../../Experiments/parallel/lshtclarge_nomad/pmlr_r1e-7_4node_48proc.log};
      \addlegendentry{LC}

%      \addplot[ultra thick, color=blue] table [x index=2, y index=6, header=true]
%      {../../Experiments/parallel/lshtclarge_singlenode/dsmlrnomad_r1e-7_p12e4_mpi1_th20.log};           
    \end{axis}
    \end{tikzpicture}
  \begin{tikzpicture}[scale=0.50]
    \begin{axis}[xmode=log, minor tick num=1,
      title={LSHTC1-large},
%      restrict y to domain=0.01:0.12,      
      xlabel={time (secs)}, ylabel={test macro F1}, ymin={0.0}, ymax={0.15}]

      \addplot[ultra thick, color=blue] table [x index=2, y index=5, header=true]
      {../../Experiments/parallel/lshtclarge_nomad/lshtcl_20e4_rep5_4node_12th.log};
      \addlegendentry{DS-MLR}
      
      \addplot[ultra thick, color=green] table [x index=2, y index=5, header=true]
      {../../Experiments/parallel/lshtclarge_nomad/pmlr_r1e-7_4node_48proc.log};
      \addlegendentry{LC}
      
%      \addplot[ultra thick, color=blue] table [x index=2, y index=5, header=true]
%      {../../Experiments/parallel/lshtclarge_singlenode/dsmlrnomad_r1e-7_p12e4_mpi1_th20.log};           
      
    \end{axis}
  \end{tikzpicture}
  \caption{(Left): test micro F1 vs time, (Right): test macro F1 vs time}
  \label{fig:plotfscore_lshtclarge}
\end{figure*}

%%%% wikilarge %%%%%%
\begin{figure*}
  \centering
   \begin{tikzpicture}[scale=0.50]
    \begin{axis}[minor tick num=1,
    legend style={at={(0.5,0.3)}, anchor= west},
      title={Wikilarge},
      xlabel={time (secs)}, ylabel={test micro F1}]

      \addplot[ultra thick, color=blue] table [x index=2, y index=6, header=true]
      {../../Experiments/parallel/wikilarge_nomad/wikilarge_res_5_np25_20_100_1e-7_15e5.log};
      \addlegendentry{DS-MLR}
      
    \end{axis}
    \end{tikzpicture}
  \begin{tikzpicture}[scale=0.50]
    \begin{axis}[minor tick num=1,
      legend style={at={(0.5,0.3)}, anchor= west},
      title={Wikilarge},
      xlabel={time (secs)}, ylabel={test macro F1}]

      \addplot[ultra thick, color=blue] table [x index=2, y index=5, header=true]
      {../../Experiments/parallel/wikilarge_nomad/wikilarge_res_5_np25_20_100_1e-7_15e5.log};
      \addlegendentry{DS-MLR}
      
    \end{axis}
  \end{tikzpicture}
  \caption{(Left): test micro F1 vs time, (Right): test macro F1 vs time}
  \label{fig:plotfscore_wikilarge}
\end{figure*}

%%%% ODP %%%%%%
\begin{figure*}
  \centering
   \begin{tikzpicture}[scale=0.50]
    \begin{axis}[minor tick num=1,
    legend style={at={(0.5,0.3)}, anchor= west},
      title={ODP},
      xlabel={time (secs)}, ylabel={test micro F1}]

      \addplot[ultra thick, color=blue] table [x index=2, y index=6, header=true]
      {../../Experiments/parallel/odp/odp_nomaddsmlr_260_np20_1_1e5_9.221e-7_packet_10.log};
      \addlegendentry{DS-MLR}
      
    \end{axis}
    \end{tikzpicture}
  \begin{tikzpicture}[scale=0.50]
    \begin{axis}[minor tick num=1,
      legend style={at={(0.5,0.3)}, anchor= west},
      title={ODP},
      xlabel={time (secs)}, ylabel={test macro F1}]

      \addplot[ultra thick, color=blue] table [x index=2, y index=5, header=true]
      {../../Experiments/parallel/odp/odp_nomaddsmlr_260_np20_1_1e5_9.221e-7_packet_10.log};
      \addlegendentry{DS-MLR}
      
    \end{axis}
  \end{tikzpicture}
  \caption{(Left): test micro F1 vs time, (Right): test macro F1 vs time}
  \label{fig:plotfscore_odp}
\end{figure*}

%%%% YouTube8M-Video %%%%%%
\begin{figure*}
  \centering
   \begin{tikzpicture}[scale=0.50]
    \begin{axis}[minor tick num=1,
    legend style={at={(0.5,0.3)}, anchor= west},
      title={YouTube8M-Video},
      xlabel={time (secs)}, ylabel={test micro F1}]

      \addplot[ultra thick, color=blue] table [x index=2, y index=6, header=true]
      {../../Experiments/parallel/youtube_video/yt8m_nomaddsmlr_260_np4_1_1e5_2.039e-7_packet_10.log};
      \addlegendentry{DS-MLR}
      
    \end{axis}
    \end{tikzpicture}
  \begin{tikzpicture}[scale=0.50]
    \begin{axis}[minor tick num=1,
      legend style={at={(0.5,0.3)}, anchor= west},
      title={YouTube8M-Video},
      xlabel={time (secs)}, ylabel={test macro F1}]

      \addplot[ultra thick, color=blue] table [x index=2, y index=5, header=true]
      {../../Experiments/parallel/youtube_video/yt8m_nomaddsmlr_260_np4_1_1e5_2.039e-7_packet_10.log};
      \addlegendentry{DS-MLR}
      
    \end{axis}
  \end{tikzpicture}
  \caption{(Left): test micro F1 vs time, (Right): test macro F1 vs time}
  \label{fig:plotfscore_youtube8m_video}
\end{figure*}

%%%% Reddit-Small %%%%%%
\begin{figure*}
  \centering
   \begin{tikzpicture}[scale=0.50]
    \begin{axis}[minor tick num=1,
    legend style={at={(0.5,0.3)}, anchor= west},
      title={Reddit-Small},
      xlabel={time (secs)}, ylabel={test micro F1}]

      \addplot[ultra thick, color=blue] table [x index=2, y index=6, header=true]
      {../../Experiments/parallel/redditsmall/small_reddit_data_newscript_260_np20_1_5e10_1e-19_packet_10.log};
      \addlegendentry{DS-MLR}
      
    \end{axis}
    \end{tikzpicture}
  \begin{tikzpicture}[scale=0.50]
    \begin{axis}[minor tick num=1,
      legend style={at={(0.5,0.3)}, anchor= west},
      title={Reddit-Small},
      xlabel={time (secs)}, ylabel={test macro F1}]

      \addplot[ultra thick, color=blue] table [x index=2, y index=5, header=true]
      {../../Experiments/parallel/redditsmall/small_reddit_data_newscript_260_np20_1_5e10_1e-19_packet_10.log};
      \addlegendentry{DS-MLR}
      
    \end{axis}
  \end{tikzpicture}
  \caption{(Left): test micro F1 vs time, (Right): test macro F1 vs time}
  \label{fig:plotfscore_reddit_small}
\end{figure*}

%%%% Reddit-Full %%%%%%
\begin{figure*}
  \centering
   \begin{tikzpicture}[scale=0.50]
    \begin{axis}[minor tick num=1,
    legend style={at={(0.5,0.3)}, anchor= west},
      title={Reddit-Full},
      xlabel={time (secs)}, ylabel={test micro F1}]

      \addplot[ultra thick, color=blue] table [x index=2, y index=6, header=true]
      {../../Experiments/parallel/redditfull/redditfull_40_250_np40_1_4e10_1e-19.log};
      \addlegendentry{DS-MLR}
      
    \end{axis}
    \end{tikzpicture}
  \begin{tikzpicture}[scale=0.50]
    \begin{axis}[minor tick num=1,
      legend style={at={(0.5,0.3)}, anchor= west},
      title={Reddit-Full},
      xlabel={time (secs)}, ylabel={test macro F1}]

      \addplot[ultra thick, color=blue] table [x index=2, y index=5, header=true]
      {../../Experiments/parallel/redditfull/redditfull_40_250_np40_1_4e10_1e-19.log};
      \addlegendentry{DS-MLR}
      
    \end{axis}
  \end{tikzpicture}
  \caption{(Left): test micro F1 vs time, (Right): test macro F1 vs time}
  \label{fig:plotfscore_reddit_full}
\end{figure*}

\end{document}